\documentclass[10pt,journal,compsoc]{IEEEtran}
\usepackage{amsmath}
\usepackage{mathrsfs}
\usepackage{bbm}
\usepackage{graphicx}
\usepackage{subfigure}
\usepackage{cleveref} 
\usepackage{mathrsfs}
\usepackage{multirow}
\usepackage{balance} 
\usepackage[vlined,boxed,commentsnumbered, ruled, linesnumbered]{algorithm2e}
\usepackage{amsfonts,amssymb}
\usepackage{array}
\usepackage{setspace} 
\usepackage{xcolor}
\usepackage{soul}
\usepackage{amsthm}
\newtheorem{theorem}{Theorem}
\usepackage{makecell}

\ifCLASSOPTIONcompsoc
 \usepackage[nocompress]{cite}
\else
 \usepackage{cite}
\fi
\ifCLASSINFOpdf
\else
\fi
\begin{document}
%
\title{OCTOPUS: Overcoming Performance and Privatization Bottlenecks in Distributed Learning}
%
%
%
%
\author{Shuo~Wang,~\IEEEmembership{Member,~IEEE,}
		Surya~Nepal,~\IEEEmembership{Member,~IEEE,}
		Kristen~Moore,~\IEEEmembership{Member,~IEEE,}
		Marthie~Grobler,~\IEEEmembership{Member,~IEEE,}
		Carsten~Rudolph,~\IEEEmembership{Member,~IEEE,}
		Alsharif~Abuadbba,~\IEEEmembership{Member,~IEEE,}
\IEEEcompsocitemizethanks{\IEEEcompsocthanksitem Shuo Wang is with the CSIRO's Data61 and Cybersecurity CRC, Australia. 
E-mail: shuo.wang@csiro.au
\IEEEcompsocthanksitem Surya Nepal, Kristen Moore, Marthie Grobler and Alsharif Abuadbba are with CSIRO's Data61 and Cybersecurity CRC.\protect\\
E-mail: \{Surya.Nepal, Kristen.Moore, Marthie.Grobler and Alsharif.Abuadbba\}@data61.csiro.au.
\IEEEcompsocthanksitem Carsten Rudolph is with Faculty of Information Technology at Monash University, Melbourne, Australia.
}
}
%
%

\markboth{IEEE Transactions on XXX, April 2021}%
{Wang \MakeLowercase{\textit{et al.}}: OCTOPUS: Overcoming Performance and Privatization Bottlenecks in Distributed Learning}
%



\IEEEtitleabstractindextext{%
\begin{abstract}
The diversity and quantity of data warehouses, gathering data from distributed devices such as mobile devices, can enhance the success and robustness of machine learning algorithms. Federated learning enables distributed participants to collaboratively learn a commonly-shared model while holding data locally. However, it is also faced with expensive communication and limitations due to the heterogeneity of distributed data sources and lack of access to global data. 
In this paper, we investigate a practical distributed learning scenario where multiple downstream tasks (e.g., classifiers) could be efficiently learned from dynamically-updated and non-iid distributed data sources while providing local data privatization. We introduce a new distributed/collaborative learning scheme to address communication overhead via latent compression, leveraging global data while providing privatization of local data without additional cost due to encryption or perturbation. 
This scheme divides learning into (1) informative feature encoding, and transmitting the latent representation of local data to address communication overhead; (2) downstream tasks centralized at the server using the encoded codes gathered from each node to address computing overhead. 
Besides, a disentanglement strategy is applied to address the privatization of sensitive components of local data. 
Extensive experiments are conducted on image and speech datasets. The results demonstrate that downstream tasks on the compact latent representations with the privatization of local data can achieve comparable accuracy to centralized learning.
\end{abstract}
}

\maketitle
\IEEEdisplaynontitleabstractindextext
%
\IEEEpeerreviewmaketitle
\IEEEraisesectionheading{\section{Introduction}\label{sec:introduction}}
The success of machine learning approaches relies on the availability of large datasets \cite{vaswani2019fast}. Distributed data sources may provide one feasible method of collecting large amounts of data. With the prevalence of the Internet of Things (IoT) and 5G networks, a large number of modern distributed devices are incorporated, resulting in a wealth of data generated in real-time. One of the common data types is high-dimensional data, e.g., high-fidelity images, such as in video surveillance. In addition, data derived from distributed devices usually contain sensitive information that can be used for re-identification. The transmission of private information from distributed devices could pose a threat to data security and privacy. Consequently, it is not feasible to collect raw data directly from distributed devices to centralized servers for analysis or model training. 

Federated learning has emerged as a training paradigm that enables distributed participants to collaboratively learn a commonly-shared model while retaining data locally, without the need to exchange the data. Federated learning enables the training procedures to be performed in distributed (collaborative) learning settings from a group of distributed data sources \cite{vogels2019powersgd,konevcny2016federateda,konevcny2016federatedb,qu2021proof}. 
Federated learning approaches, however, are constrained by performance and privatization issues. 
\newline
\textcolor{black}{
\textbf{Performance of communication.} To retain consensus across the network, distributed nodes have to frequently exchange volume gradient updates of the sizeable learned model \cite{stich2018local,khaled2020tighter,basu2020qsparse} with the server, resulting in expensive communication and scalability constraints  \cite{lin2017deep,zhang2017zipml,zhou2021communication}, especially when devices are typically communication-constrained with limited bandwidth, e.g., IoT devices. 
\newline
\textbf{Performance of computing and storage}. The computation and storage capacity of distributed nodes and servers may not be sufficient to train a complicated downstream model, especially when devices are typically resource-constrained. Besides, multiple downstream tasks need to federated learn multiple commonly shared complicated models \cite{mills2021multi} from scratch every time.
\newline
\textbf{Performance of accuracy: Heterogeneity of local data bias and data variation with time.} The federated models' accuracy may be degraded by the heterogeneity of local data bias and data variation with time. Due to the lack of global data, there is spatial heterogeneity such that each node may have a data value bias and data size variation with respect to the general population, and temporal heterogeneity that the distribution of each local dataset may vary with time. 
\newline
\textbf{Privatization Control.} Communicating model updates or collected data may contain personally identifiable information (PII) from the local data, requiring additional privacy-preserving strategies. The server can be also considered trusted or not. The lack of access to global training data can also lead to other security concerns, such as data poisoning, backdoor attacks, or unwanted biases entering the training, e.g., age, gender, sexual orientation.
}

\textcolor{black}{
Consequently, it is important to investigate a practical distributed learning scenario where multiple downstream tasks (e.g., classifiers) could be efficiently learned from dynamically updated and non-iid distributed data sources, while providing local privatization. However, existing federated learning schemes failed in this scenario due to the bottlenecks mentioned above.
}

\textcolor{black}{
To reduce the communication burden and take advantage of global data while addressing privatization in one task, we propose a novel distributed learning scheme inspired by the octopus.
Octopuses have a brain that serves as a nerve center but contains only 40 percent of their brain cells. The remaining 60 percent is found in its eight arms, which constitute eight "mini-brains".
The octopus collects environmental information through its suckers and receptors on each arm. The arms can carry out simple learning and act independently, sending abstract information to the brain for higher-level decision-making. Due to this solution, the octopus has extremely efficient cognitive capabilities beyond the average animal's. 
Similarly, our OCTOPUS learning scheme distributes feature extraction and encoding at edge nodes while gathering encoded codes and learning downstream models at the server. 
Compared to traditional federated learning and centralized learning, OCTOPUS addresses bottlenecks towards a practical distributed learning approach with multiple downstream tasks and dynamically updated and non-IID distributed data sources. It has the following advantages: 
}

\textcolor{black}{
(1) The OCTOPUS reduces communication overhead without affecting the quality of the representation. It uses a Distributed Vector Quantized Autoencoder (DVQ-AE) to enable distributed encoding at the edge nodes to extract expressive and compressed representation features transmitted as a set of low-dimensional discrete indices. 
}

\textcolor{black}{
(2) The OCTOPUS maintains accuracy performance against data heterogeneity by addressing the lack of global data. It applies Group and Sliced Vector Quantization and Flexible and Stabilized Training to handle the heterogeneity of local data bias and data variation over time and continuously update the individual DVQ-AE for each distributed participant.
}

\textcolor{black}{
(3) The OCTOPUS balances the computation and storage performance via the training of distributed DVQ-AE between the nodes and the server, and places the training of downstream models, e.g., classifiers, using latent codes at the cloud only, avoiding the complexity of the models and the constraints of the devices at the node.
}

\textcolor{black}{
(4) To address the privatization of local data, the disentanglement strategy is incorporated in OCTOPUS to separate public components from private ones automatically. 
Each distributed participant can only release the user-specific public latent codes to the server. The downstream tasks could be conducted on the gathered public components with minimal accuracy degradation compared with using the entire raw data. 
}

We conduct extensive experiments on image and speech data. The results demonstrate our approach's comparable accuracy and privatization with a smaller transmission and local privatization burden compared to ordinary federated learning and centralized learning. 

\section{OCTOPUS Learning Framework}
\subsection{Problem Formulation}
This work aims to develop a practical distributed learning scenario in which a variety of downstream tasks can be learned from dynamically-updated and non-iid distributed data sources. Intuitively, we hope to learn a global feature dictionary, and distributed local encoder to map a high-dimensional sample to the dictionary indexes, then transmit the index matrix and learn multiple downstream tasks on the gathered compressed features. The global dictionary is expected to be updated easily to address local data bias and dynamic updates. Further, plug-in privatization of local data could be achieved by forcing the global dictionary to carry only those features that are not sensitive.

There are many practical scenarios, for example, speech-to-text recognition. 
Speech consists of content (phonemes) and style (speaker identification).
It has demonstrated that the global dictionary learned for speech can be highly related to phonemes with little speaker identification information. 
Hence, the same sentence spoken by different speakers would be projected to a similar dictionary item. 
As a result, the local encoder could map speech to the index of the dictionary to transmit instead of the raw speech. Furthermore, the dynamic dictionary update could consider the spatial and temporal bias associated with local data.
Multiple downstream tasks could be conducted on the gathered compressed features via a simple model with less computation than raw data.

Formally, we consider the problem setting where a stream of gathered samples $x \sim D_i, ~ x \in \mathbb{R}^n$ is continuously and independently collected from $M$ different distributed devices $D_i,~ i = 1 , \cdots , M$, to a server for model training. 
The variety and quantity of the training data collected from distributed devices at a high frequency determine the robustness of the final model. Using this technology, it is expected to be possible to collect the distributed information without having to incur large communication overheads and without being concerned with the possibility that sensitive information $x^s$ could be derived from the transmitted data $\hat{x}$. Further, the compressed representations that reveal most of the important features of the data could address the computational and storage concerns. 

\textcolor{black}{
Consequently, we propose the OCTOPUS distributed learning scheme, which is designed to possess the following properties: 
}

\textcolor{black}{
\textbf{(a) Encoding}: transforms raw data to low dimensional latent space compressed representation features (latent codes for short) for transmission and downstream tasks, $E: X \rightarrow Z$ and $Z \in \mathbb{R}^k, ~k \ll n$, to address the communication constraints. This encoding is characterized by the conditional probability distribution $P_{Z|X}$; 
}

\textcolor{black}{
\textbf{(b) Disentanglement}:
the encoder is incorporated with an additional disentanglement strategy to facilitate the decomposition of the latent representation, $Z$, into a public component $Z_{ \bullet }$ irrelevant to PII, and a private component $Z_{\circ}$ relevant to PII, 
i.e. $ DT(Z)=Z_{\circ} + Z_{ \bullet }$, to address privatization on local data.
Specifically, we isolate sensitive and non-sensitive attributes into separate subspaces while ensuring that the latent space factorizes these subspaces independently.
}

\textcolor{black}{
\textbf{(c) Flexible and Stabilized Training}: all collected latent codes $Z$ are stored and used for training of downstream models $T_i$, e.g. classifier, at the cloud. 
The compressed version of global data could address the computing and storage overheads while taking advantage of the global data. 
Training the encoding mechanism is expected to require a minimal allocation of communication and computation resources for each distributed node while permitting continuous updates. 
}

\subsection{Overview of OCTOPUS}

\textcolor{black}{
 As shown in Figure 1, the workflow of the OCTOPUS is: 
 \newline
 \textbf{Step 1.} Learn an initial global DVQ-AE on a relevant public-available dataset at the server. This involves learning a global feature dictionary (which is composed of K items with dimension M), encoding an input into a H*W feature vector with dimension M, and finding the index of the most similar dictionary item for each feature vector, resulting in H*W index matrix as the compressed features. 
  \newline
\textbf{Step 2.} One-shot locally fine-tuning with distributed encoders and a joint decoder for the local DVQ-AE at each participant in a distributed manner based on the initial global DVQ-AE and the global dictionary. 
 \newline
\textbf{Step 3.} Based on the disentanglement strategy of the local fine-tuned DVQ-AE, each distributed participant can release only the latent codes as public components to the server, without concerns of releasing private information or incurring the additional cost of encryption or perturbation. 
 \newline
\textbf{Step 4.} Local encoder transmits low-dimensional compressed latent features of collected samples at a higher frequency. 
 \newline
\textbf{Step 5.} When new data comes in, the distributed autoencoder will be continuously updated by updating the local codebook via a simple exponential moving average, instead of retraining the local encoder and decoder, and then sent to the global dictionary at a lower frequency.
 \newline
\textbf{Step 6.} A variety of downstream tasks could be learned from the gathered compressed features via simple models at the server. The simple inference model could be returned to the node for real-time inference if required. 
}

\textcolor{black}{
OCTOPUS consists of the basic Distributed Vector Quantized Autoencoder (for general communication, detailed in Section 2.3) with three performance enhancement strategies: Group and Sliced Vector Quantization (for accuracy performance, detailed in Section 2.4), Disentanglement (for privatization performance, detailed in Section 2.5), and Flexible and Stabilized Training (for computing and storage performance and continuously updating, detailed in Section 2.6) strategies. 
The basic Distributed Vector Quantized Autoencoder consists of an encoder $E$, a vector quantization operation $VQ$, and a decoder $D$. $E$ maps the sample into latent representation. $VQ$ finds the nearest embedding from a learned codebook for each latent representation vector to generate a discrete latent code using the embedding index. $D$ decodes latent embedding back to input space. 
Data collected from each participant will be processed locally to reduce the dimension using the fine-tuned encoder to take advantage of the data locality at each node. 
The most critical features in the data are extracted via the local encoder, and low dimensional latent codes reduce the dimension of the input data. 
We learn and transmit the expressive encoded latent codes of the data associated with local privatization through the disentanglement strategy during the client stage. During the server stage, various downstream tasks, e.g., classifications, are performed on the collected latent codes, alleviating computing and storage overheads. 
If required, the privacy-preserving representation of the raw data can be reconstructed at the server using the gathered public components and replaced with private components via the fine-tuned decoder.
}
\begin{figure}[!htb]
	\centering
	\setlength{\abovecaptionskip}{-0.05cm}
	\setlength{\belowcaptionskip}{-0.2cm}
	\includegraphics[width=3.5in,height=2.6in]{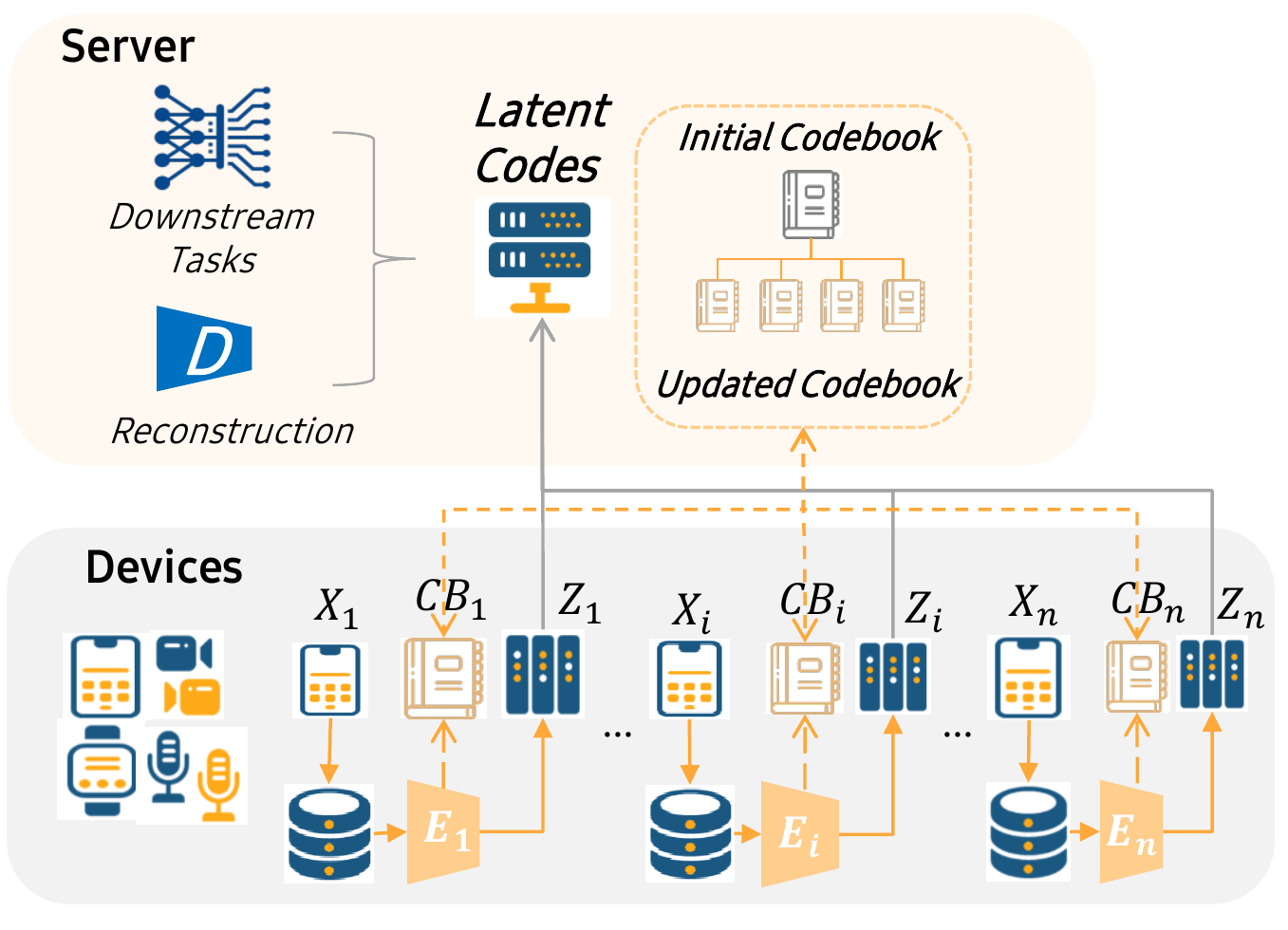}
	\caption{The OCTOPUS learning framework.}
\end{figure}
\vspace{-3mm}
\subsection{Basic Vector Quantized Procedure}
\textcolor{black}{
The basic encoding function is introduced in this section. 
Generally, the VQ-based autoencoder consists of three components, encoder $E$, decoder $D$, and codebook $C$. 
For the DVQ-AE, the encoder maps the sample $x$ into latent representation $E(x) = Z_e(x) \in \mathbb{R}^{H \times W \times M}$, where $H\times W$ is the number of the embedded feature vectors, and $M$ is the dimension of the feature vector. 
To further restrict the embedding space and compress the $Z_e(x)$ to a more low-dimensional latent representation, we apply the vector quantization operation to transfer the $Z_e(x) \in \mathbb{R}^{H \times W \times M} $ to $Z(x) \in \mathbb{R}^{H \times W}$ using a codebook $C$. 
}

\textcolor{black}{
In particular, the codebook $C$ is regarded as the common feature dictionary that is shared between encoder and decoder to restrict the embedding space, consisting of K items of M-dimensional embedding vector (called atoms), i.e., $e \in \mathbb{R}^{K\times M}$. 
The vector quantization operation finds the nearest embedding $e_i \in \mathbb{R}^M$ from the learned codebook for each $M$-dimensional vector of $z_e(x) \in Z_e(x)$, using the index $[i]$ of $e_i$ to construct the discretized latent codes $z(x) \in [K]^{H \times W}$. 
Finally, the decoder $D$ retrieves the index matrix $z(x)$ back to corresponding embeddings $z_q(x) \in \mathbb{R}^{H \times W \times M}$ according to the codebook, and decodes $z_q(x)$ back to input pixel space. 
Parameters of the encoder, decoder, and codebook are the trainable parameters. And the learning objective of DVQ-AE is given as: 
\begin{equation}
\scriptsize
{L = \left \| x-D(z_q(x)) \right \|^2+ \alpha \left \| sg[z_e(x)] - e] \right \|^2+ \beta \left \| z_e(x) - sg[e] \right \|^2}
\end{equation}
}
The first part of the loss function is a reconstruction loss $-log p(x|z_q(x))$, which represents the negative log-likelihood between the outputs of the encoder and the decoder after quantization respectively. 
The second part is codebook loss that determines VQ embedding vectors $e$ to figure out encoder results $ z_e(x)$. 
The third term is the commitment term, which makes the encoding commit to a VQ embedding vector $e$ and constrains how the VQ space is used. Here, $sg[ \cdot]$ is the stop gradient operator, maintaining its argument during the forward pass and returning zero gradients during the backward pass. The gradient of this non-differentiable step is approximated using the straight-through estimator. 
\begin{figure}[!htb]
	\centering
	\setlength{\abovecaptionskip}{-0.05cm}
	\setlength{\belowcaptionskip}{-0.1cm}
	\includegraphics[width=3.6in,height=2.6in]{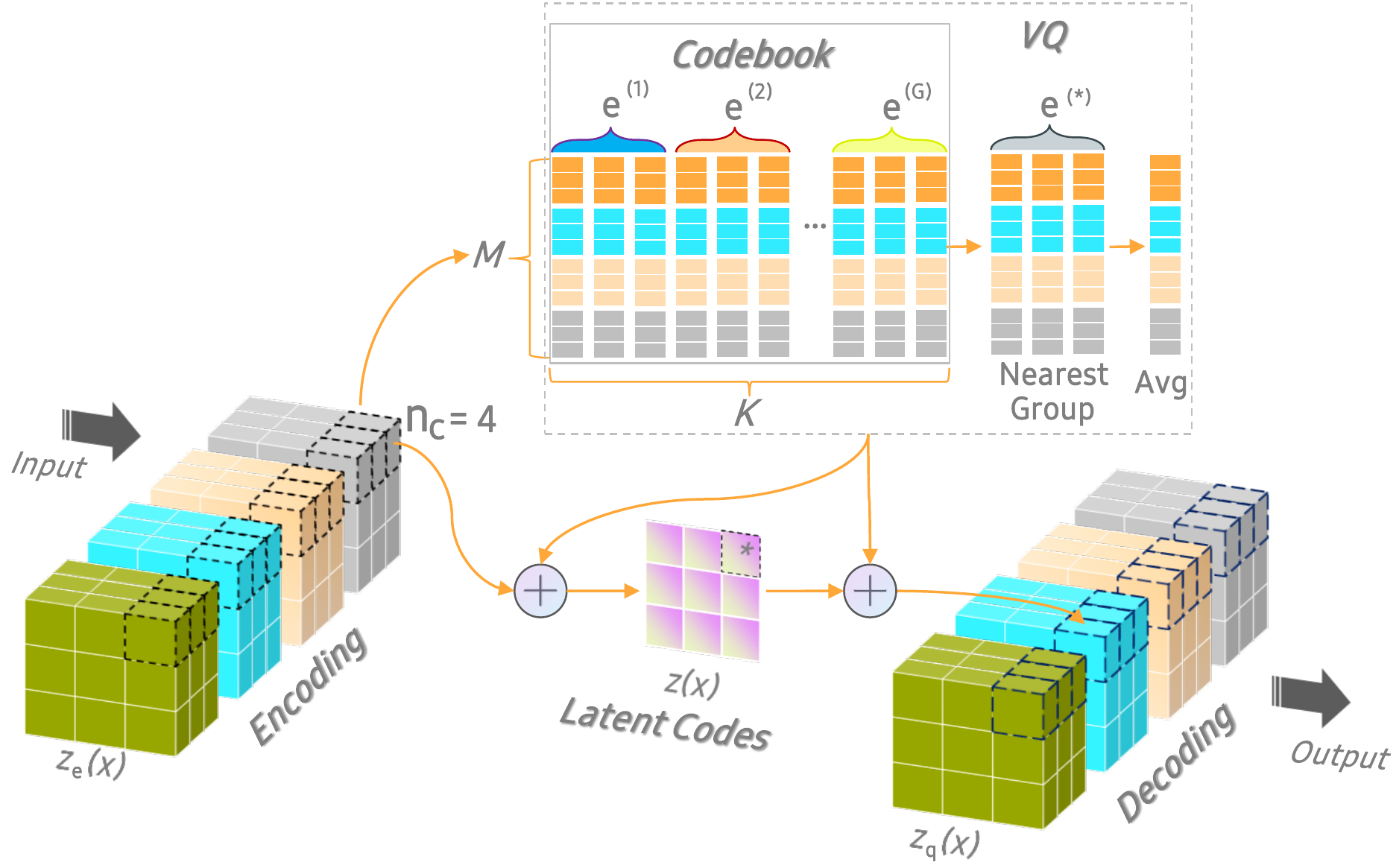}
	\caption{The GSVQ framework.}
\end{figure}
\subsection{Group and Sliced Vector Quantization (GSVQ)}
As the first enhancement of DVQ-AE, this strategy is used to increase the accuracy performance affected by local data's spatial and temporal heterogeneity. 
For general VQ-VAE, no constraints are placed on the atom distribution in the embedding dictionary, enabling atoms adjacent to one another to represent entirely different features. To mitigate the distortions caused by the mismatch of atoms, a Group Vector Quantization (GVQ) is adopted by GSVQ to decrease the probability of quantizing the encoder's output to a mismatched atom, similar to group sparse coding algorithms \cite{szabo2011online}. GVQ aims to bring similar atoms closer to each other, and different atoms further away. 
GVQ divides codebook $e\in \mathbb{R}^{K \times M}$ into $G$ groups along K dimension, and $e^{(i)}\in \mathbb{R}^{N_g \times M},~N_g= \frac{K}{G}$ denotes the $i-$th group. 
During forward-propagation, each output of the encoder $z_e(x)$ is mapped to the nearest atom group, based on the average distance over all the atoms in the group. 
\begin{equation}
\begin{aligned}
& e^*=e^{(j)}, ~ where~j=\underset{k}{argmin}\ d(z_e(x),e^{(k)})\\
& d(z_e(x),e^{(j)})=\frac{1}{M}\sum_{k=1}^M\left \| z_e(x)-e^{(j)}_k \right \|_2
\end{aligned}
\end{equation}
The corresponding $i^{th}$ index of latent code for $z(x)$ is then computed as the weighted average of the atoms in the $e^*$.
\begin{equation}
\begin{aligned}
z(x)_i=\frac{\sum_{k=1}^{M}w_ke^*_k}{\sum_{k=1}^{M}w_k},~
w_k=\frac{1}{\left \| z_e(x)-e_k^* \right \|_2}
\end{aligned}
\end{equation}
All atoms will be updated in the back-propagation, based on the loss in Equation (1), by replacing $e^*$ with $\frac{\sum_{k=1}^{M}w_ke^*_k}{\sum_{k=1}^{M}w_k}$.

In addition, Sliced Vector Quantization (SVQ)  \cite{kaiser2018fast,rakhimov2020latent} is adopted by GSVQ to improve the efficiency of nearest neighbor embedding further. Each atom of the codebook is separated into $n_c$ parts along the M dimension, i.e. $e=\{e^j \in \mathbb{R}^{K \times M / n_c }\}_{1}^{n_c}$. Accordingly, the output of the encoder $z_e(x)$ is divided along the M dimension into $n_c$ components, and the VQ procedure is conducted in terms of the separated codebook at the corresponding spatial location. Figure 2 presents the GSVQ diagram.

\subsection{Disentanglement for Local Privatization}
\textcolor{black}{
We aim to disentangle the latent representation into two components. One reflects the semantics behind a specific sensitive attribute, and the other is irrelevant to this attribute as much as possible. Namely, we assume that the latent codes of instances consist of a public component (e.g., speech-to-text recognition) and a private component (speaker identification). 
}

\textcolor{black}{
Generally, two separate encoders and sensitive classifiers are required to extract style embeddings (private component) and content embeddings (public component), respectively, e.g., \cite{huang2017arbitrary}. However, it has been demonstrated that the style encoder is redundant \cite{wu2020one,wu2020vqvc+,chen2021again}. Therefore, we adopt the single encoder scheme to reduce the size of the parameters. 
Besides, supervision over the latent factors is always limited. Therefore, we organize the training samples for the disentangled GSVQ model in groups as well, where the samples share a common specific attribute within a group. For example, the group can be a set of face images with the same facial expression among different identities or the same phonemes voice belonging to the different persons. Group supervision enables the alignment of the semantics of the data (identity and other attributes) into the learned latent representation, as a form of weak supervision that is inexpensive to collect. Note that the independent and identically distributed (iid) is not necessary in the case of grouped samples. The only supervision at training is the organization of the data into groups  \cite{bouchacourt2018multi}. 
}

\textcolor{black}{
OCTOPUS' privatization service relies on disentanglement strategies only. 
To address the privatization of the local data, we first apply two disentangled strategies to divide the latent representation into public components (e.g., phonemes of speech) and private components (e.g., speaker identification). The disentanglement strategies employed by OCTOPUS are codebook quantization and instance normalization, without the use of identifiable information classifiers or adversarial training. 
The first disentangled strategy is codebook quantization. A well-trained vector quantized model can learn some commonly-shared features as a similar series of atoms of the codebook. For example, the codebook learned for speech can be highly related to content such as phonemes \cite{chorowski2019unsupervised,wu2020one,van2017neural,williams2021learning}. The same sentence spoken by different speakers would thus be projected to a similar codebook series. 
We force the codebook as the carrier for the commonly-shared features only (public component), while learning to represent private components such as speaker identification using the information discarded by the codebook quantization, i.e., the difference between continuous space and the discrete codes. 
The second disentangle strategy is Instance Normalization \cite{ulyanov2017improved}.
To further reduce the information of the public component (content) relevant to the private component (style), the Instance Normalization is adopted as the style normalization strategy for the codebook learning. IN can normalize the style of each individual input to the standard style. }

\begin{figure}[!htb]
	\setlength{\abovecaptionskip}{-0.05cm}
	\setlength{\belowcaptionskip}{-0.1cm}
	\centering
	\subfigure[The disentanglement strategy framework]{
		\begin{minipage}[b]{0.5\textwidth}
			\includegraphics[width=1\textwidth]{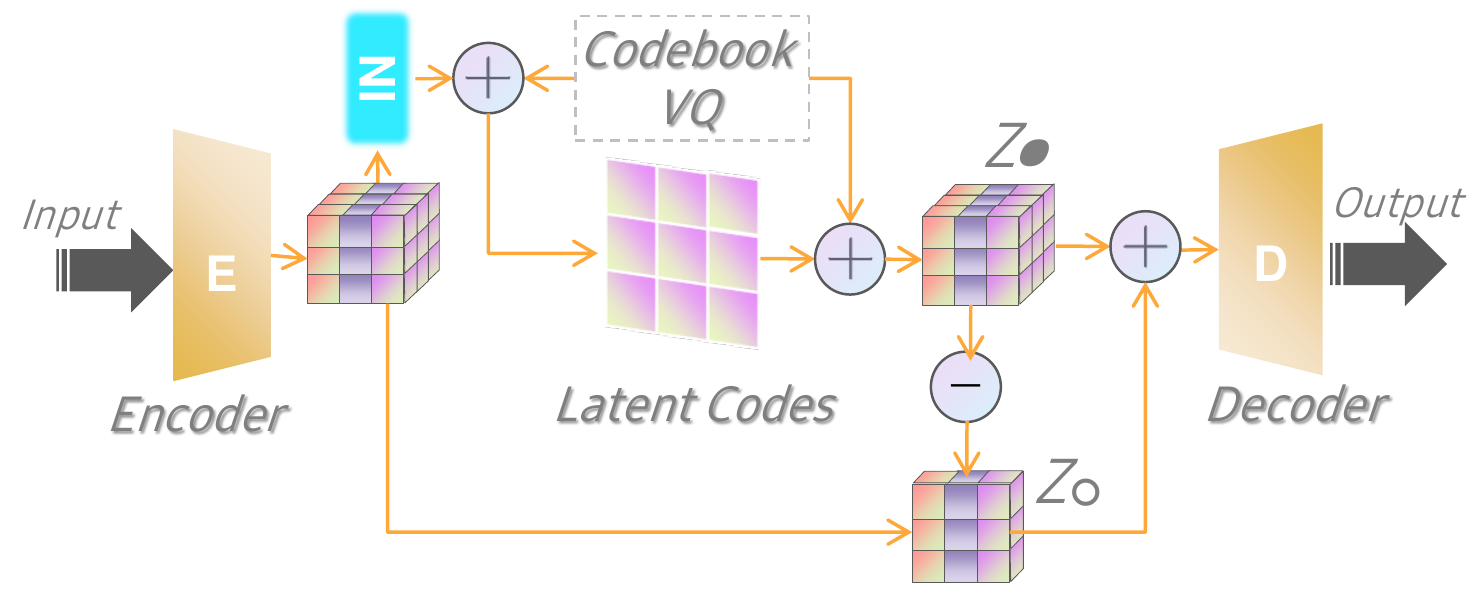}
		\end{minipage}
		\label{fig:dtf}
	}
    	\subfigure[IN normalization and VQ clustering]{
    		\begin{minipage}[b]{0.5\textwidth}
   		 	\includegraphics[width=1\textwidth]{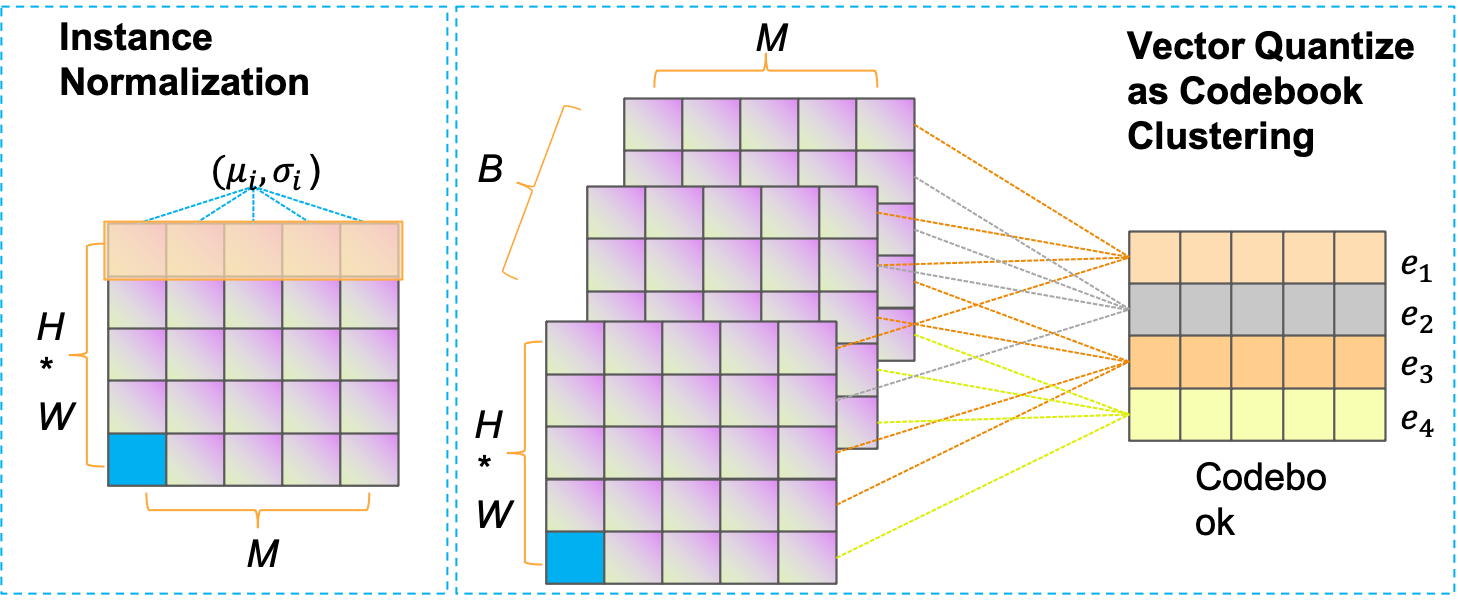}
    		\end{minipage}
		\label{fig:invq}
    	}
	\caption{The disentanglement strategy}
	\label{fig:dt}
\end{figure}

\textcolor{black}{
Given any input $x\in R^{C,H,W}$, normalized representation $IN(x)$ is computed based on the channel-wise mean $\mu$ and the channel-wise standard deviation $\sigma$ across spatial dimensions independently for each channel and each input. 
\begin{equation}
\begin{aligned}
\operatorname{IN}(x)=\gamma\left(\frac{x-\mu(x)}{\sigma(x)}\right)+\beta \\
\mu_{c}(x)=\frac{1}{H \times W} \sum_{h=1}^{H} \sum_{w=1}^{W} x^c_{hw}\\
\sigma_{c}(x)=\sqrt{\frac{1}{H\times W} \sum_{h=1}^{H} \sum_{w=1}^{W}\left(x^c_{h w}-\mu_{c}(x)\right)^{2}+\epsilon}
\end{aligned}
\label{eq:in}
\end{equation}
As the style information $ \mu$ and $\sigma$ are temporally invariant, they can be considered as private (style) representations. 
Consequently, an IN layer is added in the encoder block (before the VQ step) to filter more information about the private component from the encoded latent representation. Figure 3 shows the framework of the disentanglement strategies.
}

\textcolor{black}{
Formally, let $X^c=\{x_1, x_2, \cdots, x_T\}^c$ as a group of samples with the same sensitive class $c$, such as a sequence of acoustic features or a set of face images belonging to a specific person. Given $X^c$, the public component $Z_{ \bullet }$ and the private component $ Z_{\circ }$ is given as 
\begin{equation}
\begin{aligned}
Z_{ \bullet }= VQ(Z_e(x)), Z_{\circ } = E[Z_e(x) - Z_{ \bullet }]\\
\text { VQ }(\boldsymbol{Z_e(X))})=\left\{\boldsymbol{e}_{0}, \boldsymbol{e}_{1}, \ldots, \boldsymbol{e}_{T}\right\}, \\
\quad \boldsymbol{e}_{j}=\arg \min _{\boldsymbol{e} \in \mathcal{CB}}\left(\left\|\boldsymbol{Z_e(X)}_{j}-\boldsymbol{e}\right\|_{2}^{2}\right)
\end{aligned}
\end{equation}
Here, $ Z_{\circ }$ can be considered as the expectation of difference between $Z_e(x)$ and its vector quantized codes $Z_{ \bullet }$. Then, $Z_{ \bullet }$ is added back to $Z_{ \circ }$, fed into the decoder for the reconstruction.
}

\textcolor{black}{
Further, the latent loss $L_l$ is added to the reconstruction loss, aiming to minimize the distance between the $Z_{ \bullet }$ and the $Z_{ \bullet }$ via the IN layer. 
Accordingly, the total loss is updated to be
\begin{equation}
\begin{aligned}
L_{total} = E_{x\sim X}[\left \| D(Z_{ \bullet } + Z_{\circ }) -  x \right \|] + \\
\lambda\mathbb{E}_{t}\left[\|I N(\boldsymbol{Z_e(X)})-\boldsymbol{Z_{ \bullet }}\|_{2}^{2}\right]
\end{aligned}
\end{equation}
 }

A disentanglement strategy can be included in the initial DVQ-AE training at the server, followed by the deployment of the trained DVQ-AE to the distributed nodes. It is up to the data owner to decide whether to incorporate the sensitive components or not in the local data sharing. 

\subsection{Flexible and Stabilized Training} 
This strategy improves computing and storage performance and continuously updates local and global encoder/decoder/codebooks.
The distributed DVQ-AE is fine-tuned at each node by alternately updating the distributed encoder, joint decoder, and codebook on local data. 
Initially, the codebook is frozen for local fine-tuning and then updated at a lower frequency when local distribution changes. Local encoders and joint decoders have a fixed update frequency (e.g., one-shot fine-tuning), using local data to fine-tune.
Updating the codebook can address the drifting representation issue when the current DVQ-AE is not a good approximation for new data. The local codebook can be fine-tuned by updating the atoms of the codebook with an exponential moving average instead of the loss term \cite{van2017neural}, where it was used to reduce the variance of codebook updates. Specifically, $ \{z_{i,1},z_{i,2},\cdots,z_{i,n_i}\}$ denotes the set of $n_i$ outputs from the encoder that are closest to atom $e_i$ in the codebook so that the loss used to update the local codebook is:
\begin{equation}
    \sum_{j}^{n_i} \left \| z_{i,j}-e_i \right \|^2_2
\end{equation}
The optimal value for $e_i$ is simply the average of elements in the set:
\begin{equation}
   e_i=\frac{1}{n_i} \sum_{j}^{n_i}z_{i,j}
\end{equation}
The exponential moving average procedure is conducted as:
\begin{equation}
\begin{split}
& N^{(t)}_i := \gamma N^{(t-1)}_i+(1-\gamma)n^{(t)}_i\\
& m^{(t)}_i := \gamma m^{(t-1)}_i+(1-\gamma)\sum_j z^{(t)}_{i,j},~
e^{(t)}_i :=\frac{m^{(t)}_i}{N^{(t)}_i}
\end{split}
\end{equation}
with $\gamma=0.99$. The codebook update is conducted in a relatively lower frequency, such as monthly updates using the weekly samples $e^{(t)}_i $, then sent to the server.

\subsection{Privacy Analysis and Computational Adversary}
\subsubsection{Instance Normalization and Vector Quantization for privatization}
\textcolor{black}{
\textbf{Non-linear transformation via encoding.} 
Essentially, the encoder provides a non-linear transformation of the raw input data in pixel space to the latent feature embedding vectors. 
The training sample for the disentangled GSVQ model is divided into groups in order to perform weak supervision over the latent factors. The samples belonging to the same group share a specific characteristic, for example, a set of face images with the same facial expression among people of different identities or similar phonemes in the voice of different individuals. Here, identification is considered as style. 
The intuition of privacy-preserving for our approach is based on feature sharing instead of raw data sharing. Feature sharing is achieved via a mapping mechanism to control and reduce the inferences from shared data and filter sensitive components ( e.g., identification information) before sharing. 
Next, we explain how the IN and VQ procedures could filter sensitive components. 
}

\textcolor{black}{
\textbf{Style normalization via IN. }
It has been demonstrated that statistics on the convolutional feature of a DNN can describe the style of an image \cite{gatys2016image,li2016combining,li2017demystifying,huang2017arbitrary}. Matching many other statistics, including channel-wise mean and variance, has also been demonstrated to be effective for style transfer \cite{li2017demystifying}. 
The feature embedding from the encoder on a given input could also be considered as  $X=H \times W$ feature embedding vectors across the D channels $fd \in R^{D}$, denoted by M as a matrix of dimension $X \times D$. 
Instance normalization (IN) normalizes the feature embedding vector among the channel axes for each feature dimension. The features learned at each channel could be considered as global abstraction information, such as the R, G, B channels for images. The normalization of the channels also aims to align the feature distribution on each channel to a commonly-shared standard distribution. Therefore, the commonly shared content information (such as facial expressions) will be retained, and various style information (such as different identifications) will be generalized. 
In light of these findings, we claim that instance normalization achieves a form of style normalization through the normalization of feature statistics, specifically the mean and variance. 
IN normalization transforms each element of matrix M into a standard Gaussian distribution with $(\mu_i, \sigma_i)$. The normalized matrix M' preserves the properties of M, as a distorted version. Next, the distorted matrix M' is further shifted by multiplying it with the shifting factors $\gamma$ and $\beta$ in Eq. \ref{eq:in}, to increase the security of data.
}

\textcolor{black}{
\textbf{Codebook based clustering via VQ. }
Release latent codes z derived from the Vector Quantized procedure according to the codebook could be viewed as the cluster within a group along each feature axis. 
Based on a distance measure, the embedding vectors after the IN normalization on the feature axis $a_i \in R^{H \times W}$ in a batch of feature vectors with a size of B, denoted by $\{R^{D \times B}\}^{a_i}$, will be clustered according to each cluster center $e_j \in Codebook~CB$. This clustering procedure could also be considered as another non-linear data transformation. As a result of clustering, the features of each axis are generalized within the group, where global content information (such as facial expressions) will be retained, and style information (such as various identifications) will be generalized. 
}

\subsubsection{Against computational adversary}
\textcolor{black}{
In this work, conditional entropy is used as the privacy measurement against a computational adversary: a classifier modeled by a neural network to model the distribution $q(Y|Z_{ \bullet })$. The identifiable information classifier is used only for the post evaluation of privacy leakage after the implementation of OCTOPUS, without incorporating it into the training or implementation of OCTOPUS. 
Given random variable Y that represents the real distribution of data attribute (e.g., the identification of speaker), and a random variable Z that describes the observations of the adversary (e.g., information about messages in a communications network). 
The conditional entropy $\mathbb{H}(Y|C)$ indicates how much information is required to describe Y derived from the released latent codes. The neural network classifier trained with cross-entropy loss is a proxy for minimizing the conditional entropy of the sensitive class given the observed latent code features $Z_{\circ}$ or $Z_{ \bullet }$ \cite{agrawal2001design,diaz2007does,bouchacourt2018multi}. 
Datasets for training and testing the identifiable information classifier could be derived directly from those used for training OCTOPUS. Given the encoder of the trained OCTOPUS and an instance $x$, the training/testing input for the identifiable information classifier is a pair of public components $Z_{ \bullet }$ obtained from the encoder and the identification label for x. 
Namely, taking the released public components $Z_{ \bullet }$ as adversarial observations, the computational adversary trains a neural network classifier to model the distribution $q(Y|Z_{ \bullet })$ via minimizing the cross-entropy loss in terms of a sensitive class (attribute), such as the identification of the speaker.
\begin{theorem}
The conditional entropy of the sensitive class Y given the public component of latent code $Z_{ \bullet }$ is equal to maximize the $\mathbb{E}_{p(Y, Z_{ \bullet })}[\log p(Y \mid Z_{ \bullet })]$, which is revealed as:
\begin{equation}
\mathbb{H}(Y \mid Z_{ \bullet })=-\mathbb{E}_{p(Y, Z_{ \bullet })}[\log p(Y \mid Z_{ \bullet })]
\end{equation}
\end{theorem}
\begin{proof}
\begin{equation}
    \begin{aligned}
    &\mathbb{H}(Y \mid Z_{ \bullet }) =-\mathbb{E}_{p(Y, Z_{ \bullet })}[\log p(Y \mid Z_{ \bullet })]\\
    &=-\mathbb{E}_{p(Y, Z_{ \bullet })}\left[\log \frac{p(Y \mid Z_{ \bullet })}{q(Y \mid Z_{ \bullet })} q(Y \mid Z_{ \bullet })\right]\\
&=-\mathbb{E}_{p(Y, Z_{ \bullet })}[\log q(Y \mid Z_{ \bullet })]-\mathbb{E}_{p(Y, Z_{ \bullet })}\left[\log \frac{p(Y \mid Z_{ \bullet })}{q(Y \mid Z_{ \bullet })}\right] \\
&=-\mathbb{E}_{p(Y, Z_{ \bullet })}[\log q(Y \mid Z_{ \bullet })]-\mathbb{E}_{p(Y, Z_{ \bullet })}\left[\log \frac{p(Y, Z_{ \bullet })}{q(Y, Z_{ \bullet })}\right] \\
&=-\mathbb{E}_{p(Y, Z_{ \bullet })}[\log q(Y \mid Z_{ \bullet })]-\operatorname{KL}((p(Y, Z_{ \bullet }) \| q(Y, Z_{ \bullet })) \\
&\leq  -\mathbb{E}_{p(Y, Z_{ \bullet })}[\log q(Y \mid Z_{ \bullet })]\nonumber
    \end{aligned}
\end{equation}
As the Kullback-Leibler value $\operatorname{KL}((p(Y, Z_{ \bullet }) \| q(Y, Z_{ \bullet })) >0$ is always positive, minimizing the cross-entropy loss for training the neural network classifier is equivalent to minimizing an upper bound on the $\mathbb{H}(Y \mid Z_{ \bullet })$. 
\end{proof}
After training the classifier, the value of $\mathbb{E}_{p(Y, Z_{ \bullet })}[\log q(Y \mid Z_{ \bullet })]$ on the test dataset is used as conditional entropy in bits to measure the privacy leakage of the given sensitive attribute Y from the adversarial observation on $Z_{ \bullet }$. A lower conditional entropy value reveals less risk of privacy leakage. 
Likewise, we also measure the privacy leakage of the given sensitive attribute Y from the adversarial observation on private component of latent codes $Z_{ \circ }$ via $\mathbb{H}(Y \mid Z_{ \circ })$, via training and testing the classifier using the cross-entropy loss $\mathbb{E}_{p(Y, Z_{ \circ })}[\log q(Y \mid Z_{ \circ })]$. A higher conditional entropy value is expected as the private component is assumed to contain as much sensitive information in terms of Y as possible. The classifier is trained on a dataset consisting of M instances for each class while using $m \ll M$ instances per class to conduct the testing and conditional entropy evaluation. In addition, we also apply classifier test accuracy as an additional measure of privacy leakage.
}

\textcolor{black}{ Note that the computational adversary is only trained to evaluate the privacy, without being incorporated into the training of the OCTOPUS scheme. The privatization service of OCTOPUS is derived from disentanglement instead of adversarial training. 
To balance the trade-offs between privacy and utility, it is required to satisfy the following two properties: 
1) For the private component, the performance of an unseen attribute classifier is also substantially reduced;
2) For the remaining set of public components, the performance of an attribute classifier on released public components is close to its performance on the original instance.
}

\subsection{Overheads Analysis}
\textcolor{black}{
In this section, we conduct quantitative evaluations of the communication overheads for Octopus, compared to federated learning and split learning. 
Let $N_C$ denote the amount of clients, $N_M$ for the size of model parameters, $N_D$ for the size of data, $N_E$ for training epoch, $N_Z$ as the size of latent codes of Octopus. $N_S, ~\eta$ denote the size of the smashed layer and the fraction of parameters for the clients for split learning. 
Communication efficiency here measures the data volume transmitted across all clients and servers for training and synchronization. The data set is assumed to be equally distributed among clients for all settings for simplification. We also ignore the temporal heterogeneity that the distribution of each local dataset may vary with time. 
}

\textcolor{black}{
\textbf{Overheads of ordinary FL.} For ordinary federated learning schemes, the communication for each client is to upload gradients of parameters and then download the averaged gradient with the server for each training epoch. Then the total communication is $2\times N_C \times N_M \times N_E$. 
}

\textcolor{black}{
\textbf{Overheads of gradient compression for communication efficiency.} 
Gradient quantization and sparsification are two commonly adopted approaches to compress the uploading gradients towards the communication efficiency of FL 
\cite{sun2020adaptive,li2021talk,xu2020ternary,lin2017deep,kairouz2019advances}. Generally, the communication is conducted on some selected clients via uploading compressed gradients of parameters and then downloading the averaged gradient with the server for each training epoch. The total communication is $(N_C^{selected} \times N_M^{up} + N_C \times N_M ) \times N'_E$. 
The communication compression via quantization, sparsification or clustering  strategies could reduce the $N_C^{selected} \times N_M^{up} \times N'_E$ part. However, the compressed/quantized communication results in distortion of the gradients and a much lower learning rate, which further significantly increases the more communication rounds necessary to reach convergence, i.e., $N'_E \gg N_E$, even fail to converge.
Besides, extra complex procedures are required to handle the hard-to-converge issue, which is also not well generalized to some models such as ResNet101. For complex models (e.g., BERT), quantization-based approaches incorporate a significant decrease of accuracy. 
Furthermore, downstream update part $N_C \times N_M \times N'_E$ will not even be compressed at all.
}

\textcolor{black}{
\textbf{Overheads of split learning.} For split learning, the communication for each client, is to upload activations during forwarding propagation and download gradients during backward propagation. If considering weight sharing among clients to enhance synchronization, even with the cost of more information leakage, the total communication is $(2\times N_S \times N_D + \eta N_C \times N_M ) \times N_E$. The communication efficiency is then defined as the ratio of data transfers of federated learning and split learning, i.e., $\rho=\frac{2\times N_C \times N_M \times N_E}{2\times N_S \times N_D + \eta N_C \times N_M}$ \cite{singh2019detailed}, which could be used to decide the learning scheme according to the ratio.
}

\textcolor{black}{
\textbf{Overheads of heterogeneity.} If taking the extra strategies to address the data heterogeneity issues into consideration, besides the computation overhead and decrease of the accuracy, the direct influence for the communication is the increased $N_E$ caused by more iterations for model convergence. 
Further, if we also consider the temporal heterogeneity that the distribution of each local dataset may vary with time, a large number of iterations is necessary to fine-tune the model using new data, even retraining. 
}

\textcolor{black}{
\textbf{Overheads of multi-task scenarios.} In this work, we aim at the multi-task scenarios as well, where various models are required for massive downstream tasks. For federated learning and split learning, the entire training procedures are needed to be rerun many times, resulting in multiplied communication and computation overheads. 
}

\textcolor{black}{
\textbf{Overheads of Octopus.} 
For Octopus, the communication for each client is to upload the latent representation of the collected data and then download the trained model for once-off. Then the total communication is $ N_D \times N_Z + N_M+ \pi N_B +N_A$. Generally, the communication overhead of Octopus is far less than the federated learning and split learning. The advantages of the Octopus in the bias of overheads are summarized as follows, compared to existing FL solutions. 
 \newline
(1) \textit{\textbf{Communication size and round.}} The general communication size for Octopus is $N_Z$, which is $\ll N_M$ and $N_S$ for others. The communication round for Octopus only involves (a) once-off data collection and finalized model download, i.e., the $N_E=1$, avoiding massive communication rounds $N_E$ used for training iterations in other solutions. The initial autoencoder and the final downstream task models are trained at the server, and only once-off download communication $N_A$ is required from the server to the clients. (b) few-shot codebook updates. $N_B$ is the size of the codebook, at most, $256 \times 64$ (e.g., for 1024*1024*3 images) in our experiments, and $\pi$ is a low-frequency updating rate of the codebook, generally less than 10. 
 \newline
(2)\textit{ \textbf{Model independent and lightweight.}} As the complexity of the learned model grows, so does the communication overhead and the number of iterations for the existing FL scheme. Existing communication compression and optimization strategies may also be ineffective when dealing with complex models. In Octopus, the distributed pre-trained encoders serve as the data collector and feature extractor, allowing model-agnostic and independent communication and distributed learning, via lightweight models. 
 \newline
(3) \textit{\textbf{No/Tiny burden to handle heterogeneity.}}  Due to the global data collection, Octopus does not have to add an extra communication burden in order to handle the heterogeneity issues. Octopus uses a non-training strategy to deal with temporal heterogeneity, i.e., a simple exponential moving average. The communication overhead is only increased by few-shot updating iterations on the codebook. 
 \newline
(4) \textit{\textbf{Multi-task.}} For Octopus, there is also no uploading communication burden added to the multi-task training, and only once-off communication to download trained models to clients. As the latent representation learning could be treated as generalized feature extraction for all downstream tasks, the computation overhead is also reduced for Octopus compared to training massive models for different tasks from scratch independently. 
 \newline
(5) \textit{\textbf{Trade-offs.}} Existing strategies to handle heterogeneity, privacy, and communication are typically incorporated into the learning procedure, resulting in negative impacts on accuracy and convergence. Data collection/feature extraction and learning are separated in Octopus, which provides a better balance between accuracy and other objectives. Octopus simplifies the heterogeneity, communication efficiency, and privacy challenges into a single, less expensive task.
}

\section{Experiments}
\subsection{Datasets and Settings}
We present the results of experiments on the MNIST \cite{mnist}, CelebA \cite{liu2015faceattributes,karras2017progressive} (resize to $128 \times 128$) and Speech \cite{oord2016wavenet} datasets to demonstrate the effectiveness of our approach in terms of performance on performance and privatization. 
\textcolor{black}{
We split all samples into the training set $Tr$ (80\%) and hold out the rest as a test set $Te$ (20\%).
For the \textit{Octopus} training, 85\% of the \textit{Tr} is divided into different subsets as the distributed data sources, where the data is sorted by class, and each node receives data partition from only a single class. The remaining 15\% is used as the additional data (ATD, e.g., as the public-released relevant datasets).
For the \textit{Centralized} scenario: The set of uncompressed \textit{Tr} data is used as the entire collected data (CTD) to train the downstream tasks in a centralized manner. Then, the test set \textit{Te} is used to evaluate the performance metrics to set the baseline. 
For the \textit{Federated} scenario, the $Tr$ dataset is divided into different subsets as the distributed data sources for federated learning to train the same downstream tasks. The test set is then used to obtain the needed performance metrics. 
We mimic non-IID data partitions through dividing datasets based on data labels to evaluate performance under heterogeneous data.
In an independent and identically distributed (IID) setting, each node is randomly assigned a uniform distribution across all classes, which is the best-case scenario of the non-IID setting. We consider the worst-case of the non-IID setting, where the data is sorted by class, and each node receives data partition from only a single class. Additionally, we also control the skewness through varying the proportion of data that are non-IID to demonstrate the effect of varying the skewness. We apply 20\% non-IID as the moderate-case of the non-IID settings, in which 20\% of the dataset is separated in terms of labels, while the remaining 80\% is divided uniformly at random. 
To make a fair comparison, we also consider the additional data(ATD, e.g., as the publicly-released relevant datasets) as the globally shared between all the client devices \cite{zhao2018federated}.
}

\textcolor{black}{
The same notations for Federated Averaging (FedAvg) algorithm in \cite{mcmahan2017communication} are adopted. We apply 100 clients, each running one local epoch on top of 100 global communication rounds.
We also consider $FedProx$ \cite{li2018federated} as the approach to accelerate tuning hyperparameters for algorithms of FedAvg in the case of federated learning with personalization. 
We also evaluate the performance when the federated learning is incorporated with differential privacy to address the privacy of personally identifiable information, at $(\epsilon, \delta) = (10, 10^{-5})$-DP. 
}

\textbf{OCTOPUS  scenario:}  \textcolor{black}{
The ATD data here is the additional data that we assume could be obtained from a public-released relevant dataset. The server uses it to initialize the DVQ-AE, instead of sharing globally. } The next step is to fine-tune the local model or codebook using the distributed data source (used in the federated scenario). Each sample from every distributed data source is mapped to the node-specific autoencoder's encoder. The server gathers the encoded latent codes to be the training set to train the same downstream tasks. After that, the model's performance evaluations are obtained using the encoded version of the test set. 

In our experiments, we explore four group setups for the codebook structure with $N_v=1,4$ and $N_h=1,41$, respectively. Let the codebook size $B_x$ reveal the compression size, namely the size of compressed communication, and each group setup is conducted on  $B_{32}$,  $B_{64}$,  $B_{128}$,  $B_{256}$, and  $B_{512}$ settings. 
\subsubsection{Evaluation Settings and Metrics}
\textcolor{black}{
The following experimental settings and quantitative metrics are used to evaluate the effectiveness of the OCTOPUS framework.
\newline
(1) \textbf{Downstream task evaluation.} The first setting is the test accuracy of the downstream task models obtained after the distributed training process. 
In this setting, we consider learning a specific classifier of downstream tasks with the same neural network structure as the common target task for centralized learning (with/without differential privacy), federated learning (under IID, non-II, FedProx, data sharing, and differential privacy settings), and OCTOPUS. Generally, these schemes are performed by (i) a feature extractor, which derives embedding from raw data, e.g., three Conv1d layers with 256 hidden units, and (ii) a downstream classification function based on the learned embedding, e.g., one fully connected layer with SoftMax activation function. 
We assume the server's downstream tasks are attribute classification for MNIST (containing a circle or not) and CelebA (gender and smiling), and phoneme identity accuracy for SPEECH. 
For OCTOPUS, latent codes collected from distributed nodes via the pre-trained encoders are used to train the classifier at the server for classification or text recognition. 
Classification accuracy is used for images and Word Error Rate (WER) for speech. 
Experimental analysis is given in Section \ref{sec:evadt}. We also evaluate multi-tasks performance on the latent codes for OCTOPUS in Section \ref{sec:evacdoe}. 
\newline
(2) \textbf{Identifiable information evaluation.} Our second experimental setting is to evaluate the leakage of sensitive information from the commonly trained and shared models derived from centralized, federated, and OCTOPUS learning schemes. 
In this setting, we consider learning a specific classifier of PII information with the same neural network structure as the common target task for centralized learning (with/without differential privacy), federated learning (with IID as the best-case and with differential privacy), and OCTOPUS. Generally, these schemes are performed by (i) a feature extractor, which derives embedding from raw data, e.g., three Conv1d layers with 256 hidden units, and (ii) a downstream classification function based on the learned embedding, e.g., one fully connected layer with SoftMax activation function. 
The PII information on the server is identity recognition for MNIST (digit number itself), and CelebA (personal identification), and SPEECH (speaker identification). 
The training data for centralized and federated settings are the raw data samples. For OCTOPUS, latent codes collected from distributed nodes via the pre-trained encoders are used to train the classifier at the server for classification or text recognition. 
The classification accuracy of the identifiable attribute is used to evaluate privatization. Experimental analysis is given in Section \ref{sec:evap}. 
\newline
(3) \textbf{Disentanglement evaluation.} Taking the same identifiable information classifier, we test the recognition accuracy of identification using a set of samples with identifications excluded in the training and testing dataset on various sizes of embedding codes, as the degree of disentanglement metric for OCTOPUS. Experimental analysis is given in Section \ref{sec:evad}. 
\newline
(4) \textbf{Performance and complexity evaluation.} The training and testing time and the compression sizes are also investigated as evaluation metrics. Experimental analysis is given in Section \ref{sec:evaother}, \ref{sec:evasetting}  and \ref{sec:evatime}. 
}
\newline
The structure and hyperparameters of the encoder of DVQ-AE are given in Appendix A. 
The design, training, and testing of the deep learning models (Autoencoders and CNN Classifiers) were implemented using the Keras deep learning framework on TensorFlow, running on NVIDIA Tesla P100 GPU. 

We consider the premise of the initial global estimation only (cold start problem). As an additional advantage, we prefer to use the pre-trained model as the starting point, and then only fine-tuning is required for our scheme. The speech pre-trained initial model, for example, can be used in most speech recognition scenarios.
We have a validation subset for the initial global model training. Most hyperparameters are learned at initial learning and directly used for the following procedures.

\subsection{Evaluation on Test Accuracy on Downstream Tasks}
\label{sec:evadt}
We evaluate downstream task models' testing accuracy when trained and tested in Centralized, Federated, and OCTOPUS scenarios with MNIST, CelebA, and Speech datasets. 

Figure 4 compares the downstream task models' testing accuracy in the OCTOPUS scenario using various compression sizes compared to centralized and federated scenarios.  
Testing accuracy for the centralized scenario (uncompressed data), the baseline of accuracy, is the highest among all cases, as all features in the raw data can be used to train models. 
Additionally, the testing accuracy for OCTOPUS is higher than that of Federated in both non-IID and privacy-preserving settings on all three datasets. Also, it has been found that the more complex the image, the more significant the gap. 
As shown, there is a general degradation in the testing accuracy when the compression size increases from $B_{512}$ to $B_{32}$, as the number of features utilized to train models is decreased. The smaller the compression size, the more features the model has to train a stranger model. 

\textcolor{black}{
The rate of reduction of the model's testing accuracy varies across the models on two image datasets, with MNIST being the highest of all the compression sizes. The reason is that the number of features for simple data of small dimensions is even smaller than encoded features. The significant degree of degradation is also visible in Speech, since the complex downstream task requires more features to train the model. 
We further demonstrate that there is significant degradation of the utility for federated learning in the non-IID scenarios (approximately 10\%-30\% drop in the downstream task accuracy for the worst-case and moderate-case of non-IID settings) and after incorporating the privacy-preserving procedure (approximately 30\%-50\% drop in the downstream task accuracy). The OCTOPUS avoids such degradation by incorporating advantages from both centralized learning and federated learning. 
}

\textcolor{black}{
As differential privacy (DP) is the commonly adopted extra privacy-preserving approach used in FL, we also demonstrate the utility of the DP-based FL scheme after applying the perturbation to the communication. 
For the MNIST dataset, as shown in Figure 4, original samples can easily be recognized with more than 99\% accuracy on average for the centralized scenario. However, the DP-based privacy-preserving centralized case appears to have lower accuracy than expected in comparison to the centralized scenario. 
Even under $B_{32}$ compression size, the accuracy of OCTOPUS is generally better than that of centralized learning with DP and federated in non-IID and with DP cases. 
For the CelebA dataset, the classifier trained on the centralized case can achieve around 90\% test accuracy on both utility tasks (gender and smiling) on average, while our trained downstream task using OCTOPUS can achieve a comparable accuracy (87\% on average) on the compact latent representations. The accuracy of OCTOPUS is better than that of centralized learning with DP (82\% on average) and federated in non-IID (51\% for the worst-case and 56\% for the moderate-case) and with DP (30\% on average) cases. 
The evaluation on Speech confirms that OCTOPUS does not degenerate the utility of the expected downstream task too much compared to the centralized scenario and is better than the federated scenario. 
}

\textcolor{black}{
Generally, applying FedProx for tuning hyperparameters could enhance the performance (accuracy of the tasks) of FedAvg under non-IID settings by around 10-20\% on average compared to the worst-case of non-IID scenarios, but still less than the best-case. Besides the extra cost for conducting these improvements, it is also required to apply further privacy-preserving approaches, such as differential privacy, resulting in accuracy degradation. 
After applying data-sharing strategy \cite{zhao2018federated} to improve FedAvg with non-IID data via globally sharing the ATD data among all the client devices, the test accuracy can be increased by 10-30\% on the image datasets compared to the worst-case of non-IID scenarios, though still less than the best-case as well.
Even for the best-case of the non-IID case ( uniform distribution over all classes), the accuracy of the federated learning decreases by 30\%. 
For federated learning with hyperparameter tuning and data-sharing strategies, the accuracy enhancement cannot defeat the degradation after deploying privacy-preserving mechanisms (e.g., differential privacy).
}
\begin{figure*}[!htb]
	\centering
	\setlength{\abovecaptionskip}{-0.05cm}
	\setlength{\belowcaptionskip}{-0.1cm}
	\includegraphics[width=6.5in]{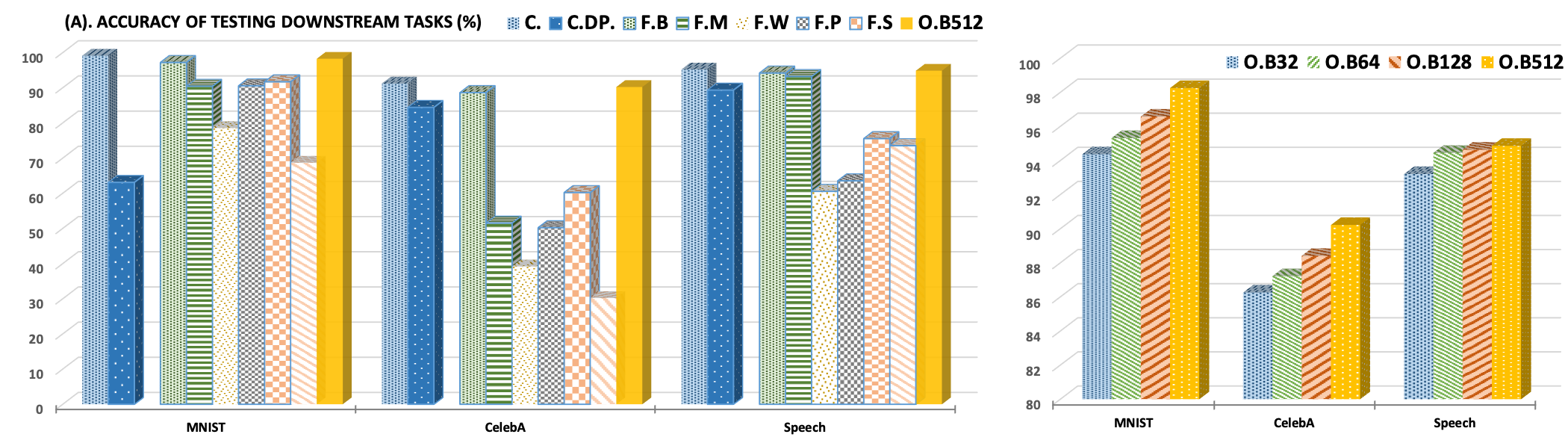}
	\caption{The comparison of the downstream task models' testing accuracy on centralized (Cen. and with differential privacy C.DP.), federated (with IID Fed.I as the best-case, with worst-cased and moderated case of non-IID settings F.W. and F.M, with FedProx F.P, with data sharing F.S, and best case with differential privacy F.N.DP.) and our methods with various size of embedding codes, respectively.}
\end{figure*}

\subsection{Evaluation on Privatization}
\label{sec:evap}
\textcolor{black}{
In order to protect FL communication privacy, additional privacy-preserving strategies such as secure computation, secure aggregation, and differential privacy schemes are usually necessary, resulting in additional computational overhead or reduced utility. In this section, we demonstrate the privatization ability of the OCTOPUS with no extra privacy-preserving mechanisms incorporated. 
We evaluate the privacy leakage via the recognition accuracy of the sensitive information and the conditional entropy in bits on the classifier testing set. 
}

\textcolor{black}{
The first metric we evaluate privatization is the recognition accuracy of sensitive information. 
For both FL and centralized settings, the identity classification accuracy remains high, which indicates that a significant amount of sensitive information has been collected from raw data during training. Adopting privacy-preserving technologies (e.g., differential privacy) could reduce the accuracy of identity classification, however, the values are still high (around 70\%). 
For the MNIST datasets, the digit number itself is considered as the private component. The task of downstream recognition is of determining whether the digit contains a circle (i.e., 0,6,8,9) or not. 
Figure 5 (a) reports the recognition accuracy of the sensitive information as a privacy evaluation metric under the settings given in Section 3.1.1. 
Classification accuracy for the private component via OCTOPUS decreases rapidly, from 95\% to below 10\%, while classification accuracy for circle recognition remains similar to that of the centralized scenario. It follows that the OCTOPUS does not degrade the utility of the expected downstream task too much while also significantly limiting the leakage of PII. 
In the CelebA dataset, we consider the identification as a privacy component, and gender and smiling recognition as downstream tasks utilizing the model in \cite{lu2017fully,torfason2016face}. 
As shown in Figure 5, the OCTOPUS model can reduce the test accuracy of private information (identification) from 85\% down to around 10\% on average, demonstrating the ability to protect private information. It demonstrates that the filtered data can still serve the desired classification tasks at the cost of only a small drop in utility accuracy (4\%). 
}

\textcolor{black}{
Considering differential privacy is associated with a lower computational cost to protect privacy via perturbation, we compare the performance between sensitive information perturbation and sensitive information filtering/replacing in terms of privacy protection. Specifically, we compare the privatization performance between OCTOPUS and centralized/federated learning with differential privacy. As shown in Figure 5, the inference probability of personally identifiable information for OCTOPUS (between 10-20\%) is much better than the other (between 30-80\%), avoiding introducing additional cost due to perturbation. The results demonstrate that sensitive filtering/replacement can achieve better privatization than perturbation. 
}

\textcolor{black}{
To further show the reconstruction with different private component settings and the possibility of generating anonymous copies and style transformations, we demonstrate the reconstructed samples in Figure 6. 
First, the downstream tasks and privacy-preserving evaluations are conducted only on the public component (the $Z_{ \bullet }$), which is derived from the pre-trained encoder and VQ module. Reconstruction is not required for these downstream tasks and privacy-preserving evaluations. 
During the reconstruction, both the public component $Z_{ \bullet }$ and the private component $Z_{ \circ }$ are required to decode back to pixel space (images), which could be considered as sampled points from the learned distribution. Since the private component contains sensitive information, this is the reason we filter them from the released features. When the public component $Z_{ \bullet }$ is combined with empty or totally random private component $Z_{ \circ }$, the decoder would produce a blurry reconstruction, as the decoder is trained on the combination of both $Z_{ \bullet }$ and $Z_{ \circ }$ from learned distribution. 
Also, it would be interesting to know what would happen if an adversary has the ability to feed private components into reconstruction, particularly those derived from publicly available training data.
Therefore, we demonstrate two possible style transformations with available private components $Z_{ \circ }$ in Figure 6: (1) perturbed $Z'_{ \circ }$ by adding random noise to the extracted private component, and (2) replacing with a different private component $Z''_{ \circ }$ derived from publicly available training data, such as ATD data. 
After reconstruction with perturbed $Z_{ \circ }$ , the sensitive information is sufficient to fool a well-trained classifier and to be difficult for humans to recognize, as demonstrated in Figure 6 (a). It is possible that a tiny perturbation on the private component could blur the digit identity while retaining the circle recognition information for MNIST. 
The replacing approach could be considered as a strong case for the perturbation approach. For example, we might use the private component (or further perturbed version) of fake face images to replace the original private component of the sensitive raw data. 
In Figure 6 (b), it is challenging to obtain sensitive information (e.g., human identification) for both machine and human perception after replacing the private component (extracted from reference samples from ATD) during reconstruction. This shows that the reconstruction with private components replaced maintains the public features such as facial expression but differs in the identification. 
For the Speech dataset, we consider the speaker's identification as the privacy component, the phoneme identity as the downstream task used the model in  \cite{oord2016wavenet}. The evaluation of Speech confirms that OCTOPUS significantly reduces the privacy leakage of the PII. 
}

\textcolor{black}{
\textbf{Conditional entropy as privacy.}
After this, we evaluate the privacy leakage of both private and public components in terms of the conditional entropy in bits on the classifier testing set, i.e., the results of the classifier trained on transmitted private, public, and both components. 
As shown in Figure \ref{fig:ceacc}, taking MNIST and Face data, for example, the performance of the classifier trained on the private components of latent codes has a lower conditional entropy value and higher test accuracy than the fully released latent codes (private + public) and the public components. It reveals that the private components of the latent codes $Z_{\circ}$ are more informative about the sensitive label. Additionally, the performance of the classifier trained on the public components of latent codes has a higher conditional entropy value and lower test accuracy than the fully released latent codes and the public components. It means the public components of the latent code $Z_{bullet}$ are uninformative in terms of the sensitive class label. 
}

\begin{figure}[!htb]
	\centering
	\setlength{\abovecaptionskip}{-0.05cm}
	\setlength{\belowcaptionskip}{-0.1cm}
	\includegraphics[width=3.4in]{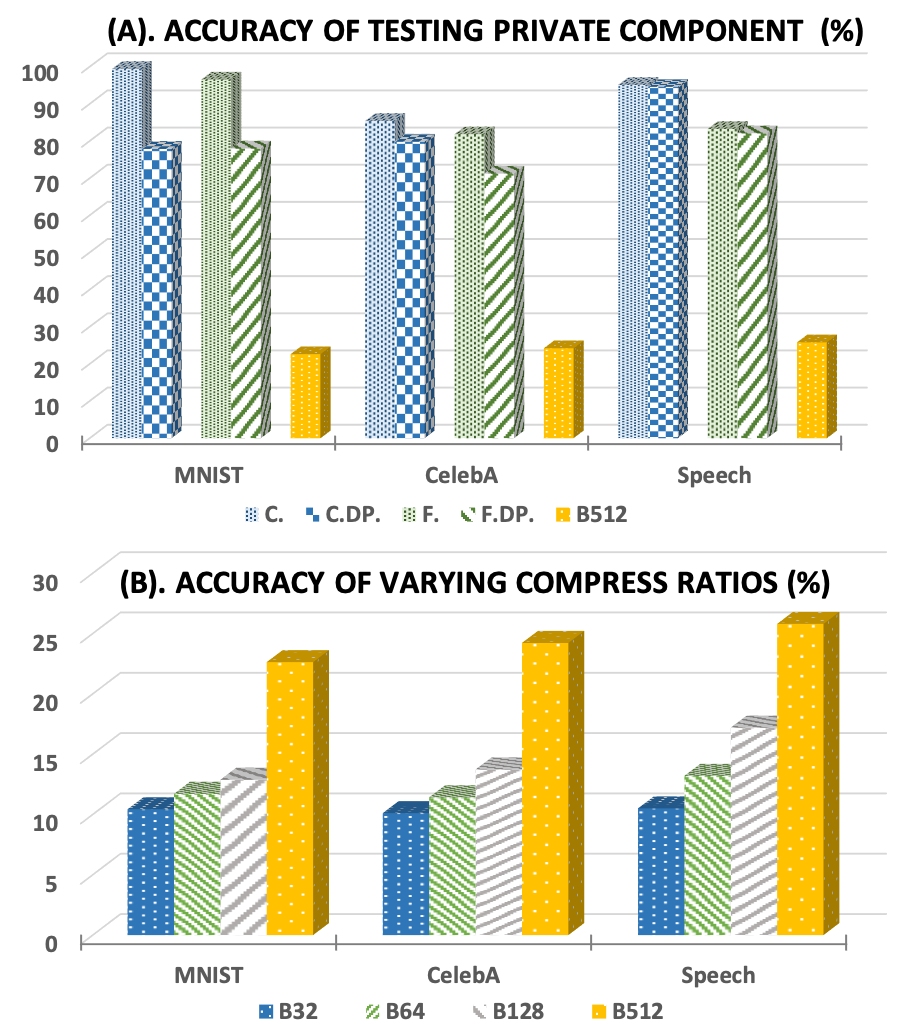}
	\caption{The comparison of the privacy components' testing accuracy on  centralized (Cen. and with differential privacy C.DP.), federated (with IID as the best-case F. and with differential privacy F.DP.) and our methods with various size of embedding codes, respectively.}
\end{figure}
\begin{figure}[!htb]
	\setlength{\abovecaptionskip}{-0.05cm}
	\setlength{\belowcaptionskip}{-0.1cm}
	\centering
	\subfigure[Original MNIST samples (left) and the corresponding reconstructed version with perturbed private components (right)]{
		\begin{minipage}[b]{0.45\textwidth}
			\includegraphics[width=1\textwidth]{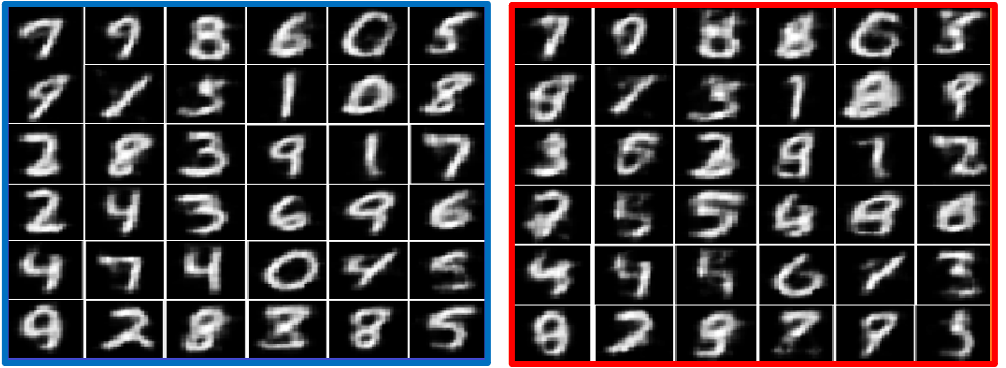}
		\end{minipage}
		\label{fig:dtf}
	}
    	\subfigure[Original face samples (blue) and the corresponding reconstructed version (red) with replacing and perturbed private components extracted from reference samples (green)]{
    		\begin{minipage}[b]{0.45\textwidth}
   		 	\includegraphics[width=1\textwidth]{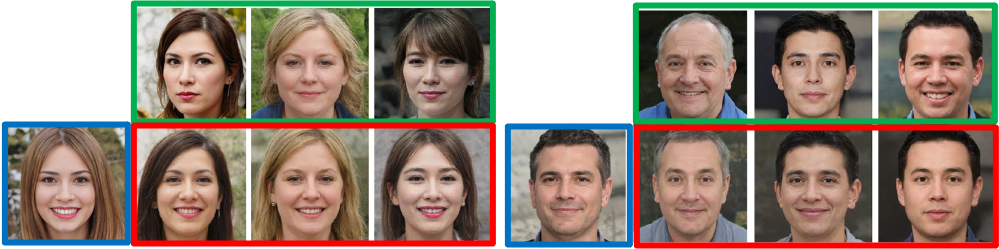}
    		\end{minipage}
		\label{fig:invq}
    	}
	\caption{Reconstruction demonstration with both public components and perturbed private components via DVQ-AE.}
	\label{fig:dt}
\end{figure}
\begin{figure}[!htb]
	\centering
	\setlength{\abovecaptionskip}{-0.05cm}
	\setlength{\belowcaptionskip}{-0.1cm}
	\includegraphics[width=3.5in,height=3.4in]{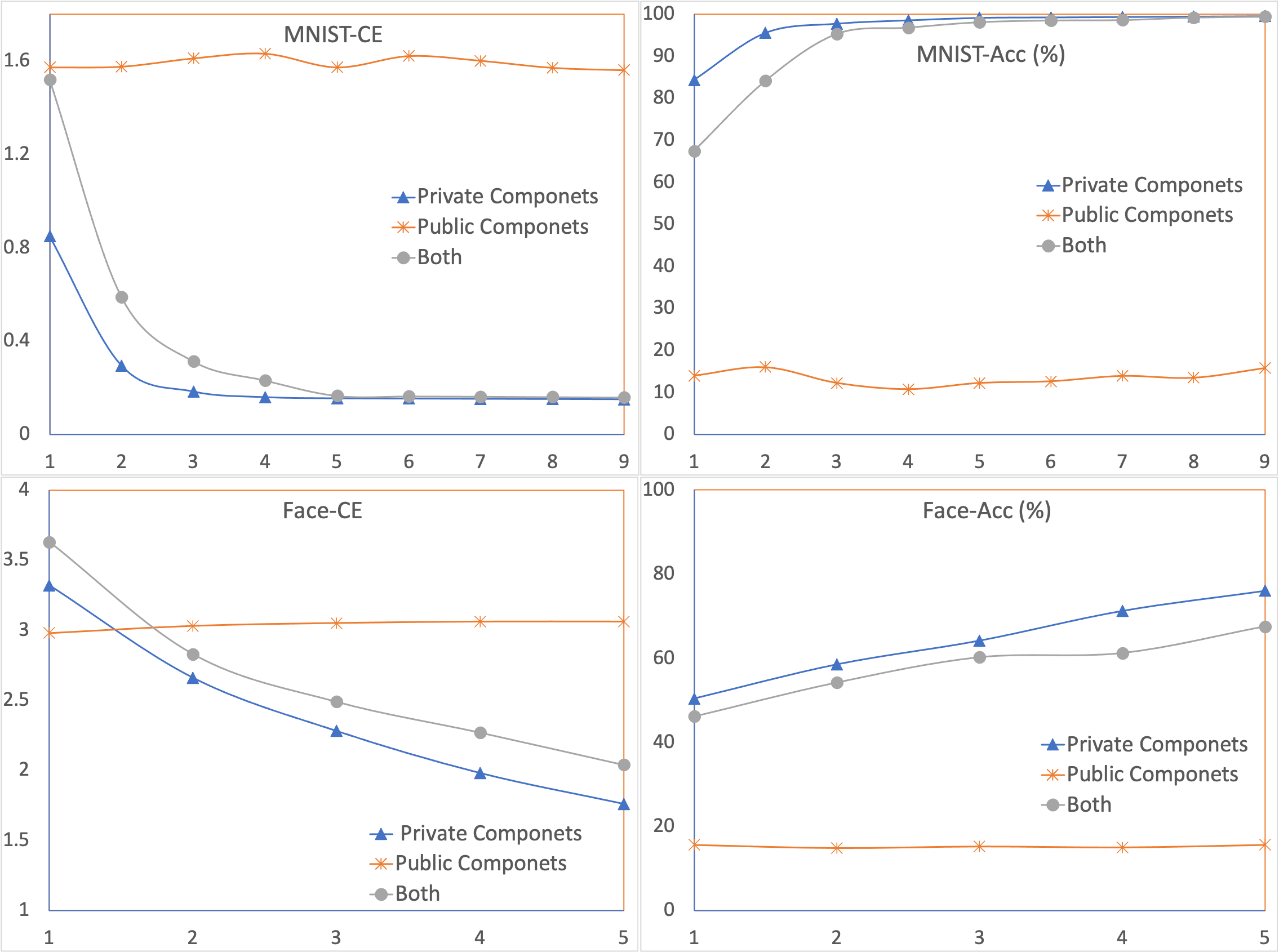}
	\caption{Evaluation on the privacy metrics.}
	\label{fig:ceacc}
\end{figure}

\subsection{Evaluation on Disentanglement}
\label{sec:evad}
Taking the Speech recognition scenario as an example, we test the privacy component's recognition accuracy using speaker voice excluded from the training dataset, as the degree of disentanglement of OCTOPUS. The results are summarized in Figure \ref{fig:distangle} and Table \ref{tab:disent}. 
The test accuracy of speaker identification is reduced by half on the codebook size for OCTOPUS, compared to the same scheme without our disentanglement strategies. This demonstrates that the latent codes would cause privacy leakage, e.g., speaker information. Namely, the fewer codes in the codebook, the less the privacy component is incorporated. 
The recognition accuracy increases, as the codebook size is enlarged from $32$ to $512$. However, the re-identification accuracy remains at low probability, even with the $B_{512}$ codebook size. 
We find a trade-off between the disentanglement and utility of codes for downstream tasks. 
If we compress the codebook into only a few codes, the disentanglement ability is strong while being hard to contain sufficient features of content information. In our experience, $B_{256}$ is a proper size of codebook to balance the trade-off. 
The experiments demonstrate that the OCTOPUS scheme can achieve good disentanglement performance without explicitly imposing any objectives or constraints on the encoder or decoder.

\begin{figure}[!htb]
	\centering
	\setlength{\abovecaptionskip}{-0.05cm}
	\setlength{\belowcaptionskip}{-0.1cm}
	\includegraphics[width=3.5in]{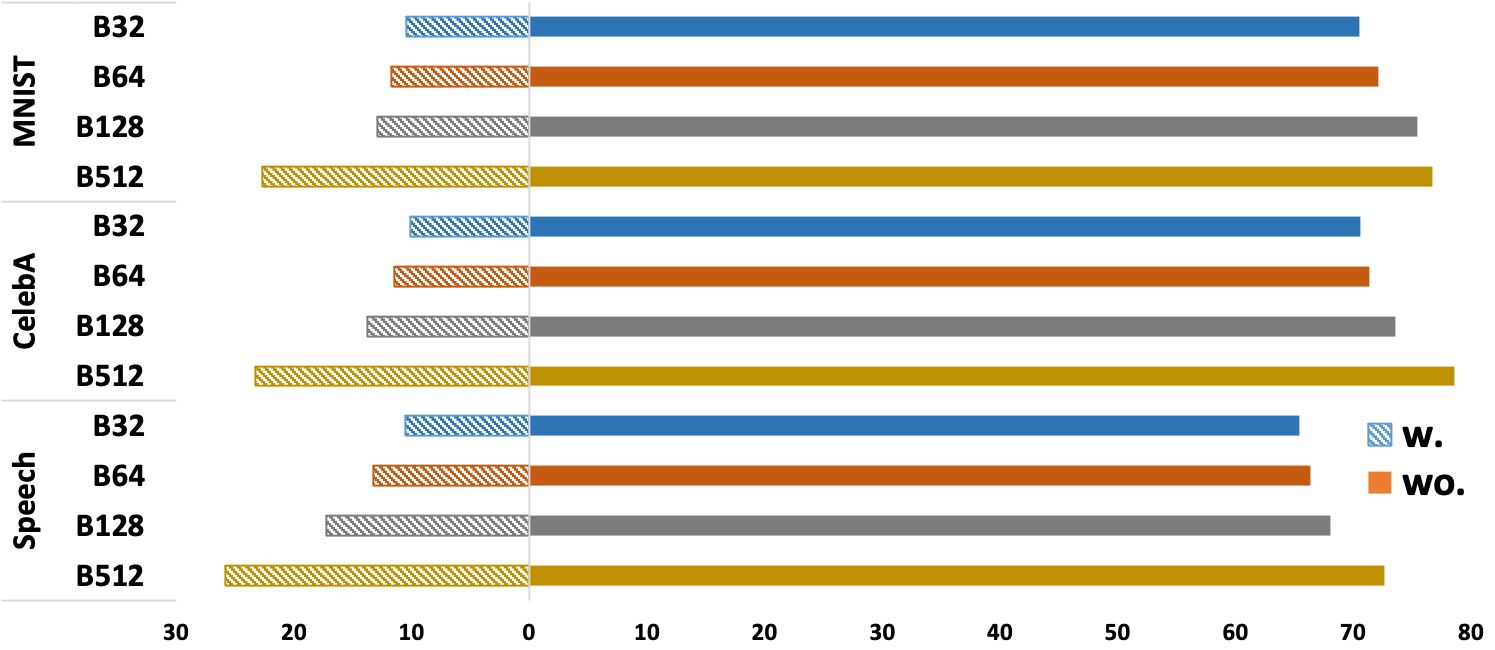}
	\caption{Accuracy of identifying private component on various size of embedding codes w/wo disentanglement strategy.}
\label{fig:distangle}
\end{figure}
\begin{table}[!htb]
\centering
	\setlength{\abovecaptionskip}{-0.05cm}
	\setlength{\belowcaptionskip}{-0.1cm}
\caption{Accuracy of identifying private component on various size of embedding codes with and without disentanglement strategy.}
\label{tab:my-table}
\scriptsize
\begin{tabular}{c|llll}
\hline
w/wo Disent. & \multicolumn{1}{c}{$B_{32}$} & \multicolumn{1}{c}{$B_{64}$} & \multicolumn{1}{c}{$B_{128}$} & \multicolumn{1}{c}{$B_{512}$} \\ \hline
MNIST  & 10.43/70.56 & 11.71/72.12 & 12.88/75.45 & 22.63/76.72 \\
CelebA & 10.11/70.65 & 11.42/71.41 & 13.73/73.64 & 23.22/78.62 \\
Speech & 10.54/65.42  & 13.22/66.36  & 17.21/68.11 & 25.82/72.64 \\ \hline
\end{tabular}
\label{tab:disent}
\end{table}
\subsection{Evaluation on Other Performance Improvements}
\label{sec:evaother}
The improvements in communication, computing, and storage performance are intuitive. For example, DVQ-AE can compress the high-dimensional raw face image with $128 \times 128 \times 3$ input size to at most $32 \times 32$ latent codes for communication, computing, and storage, achieving comparable accuracy performance as using the raw data. In the speech, the latent space is 64 times smaller than the original waveform input. In addition, we evaluated the normalized amount of time required for training the downstream models for MNIST, CelebA, and Speech datasets. We observed training time decreases across different compression sizes and models. Such reduction may be caused by the reduction of the amount of the input data and smaller number of parameters to be learned when applying the OCTOPUS algorithm with a suitable compression size. 
Detailed evaluations are provided in the following sections.

\subsection{Evaluation of the Multi-tasks on the latent codes}
\label{sec:evacdoe}
As there are clear annotations on CelebA dataset, we select 20 annotations to demonstrate the performance of the Octopus on multiple downstream tasks, e.g., binary attribute classification with respect to Hairline, Eyeglasses, Male, Mouth, Smiling, Hair Style, etc. 
For Octopus, as the collected latent codes can be considered as the output derived from a generalized feature extractor, we apply a simple linear classifier to conduct multi-task classification on the latent codes with a 512 size. 
Specifically, the input of the linear classifier is the collected latent codes, followed by three linear classifier layers are (512, 512), (512, 128), and (128, 40) sequentially, and the Sigmoid activation function, with around 0.3M parameters. 

\textcolor{black}{
The baseline approaches are three competitive deep neural network approaches, i.e., LNet+ANet, LNet+ANet with pre-training, and the MobileNets model. 
LNets + ANet \cite{liu2015deep} applies two DNNs to recognize faces and then detect the facial attributes, where both LNeto and LNets have network structures similar to AlexNet (10 convolutional layers), and ANet has four convolutional layers and one fully-connected layer, with around 10M parameters. 
MobileNet \cite{howard2017mobilenets} consists of 28 depth-wise convolutional layers and point-wise convolutional layers with around 4M parameters. 
To demonstrate the performance of Octopus, we allow the baseline models to be trained on the centralized raw data and select the best performance of the binary classification models. The dataset is divided with the ratio of 7:2:1 for training, validation, and test. All these models are trained with a learning rate of 0.001 and for 50 epochs. 
The results are reported in Figure \ref{fig:multi}, the performance of Octopus is equivalent to or better than the other three competitive models while using far less computation and parameters. This increased accuracy is due to the fact that our autoencoder functions as an enhanced feature extractor and the Vector-Quantized procedure also reduces noise in the data. 
}

\begin{figure}[!htb]
	\centering
	\setlength{\abovecaptionskip}{-0.05cm}
	\setlength{\belowcaptionskip}{-0.1cm}
	\includegraphics[width=3.5in]{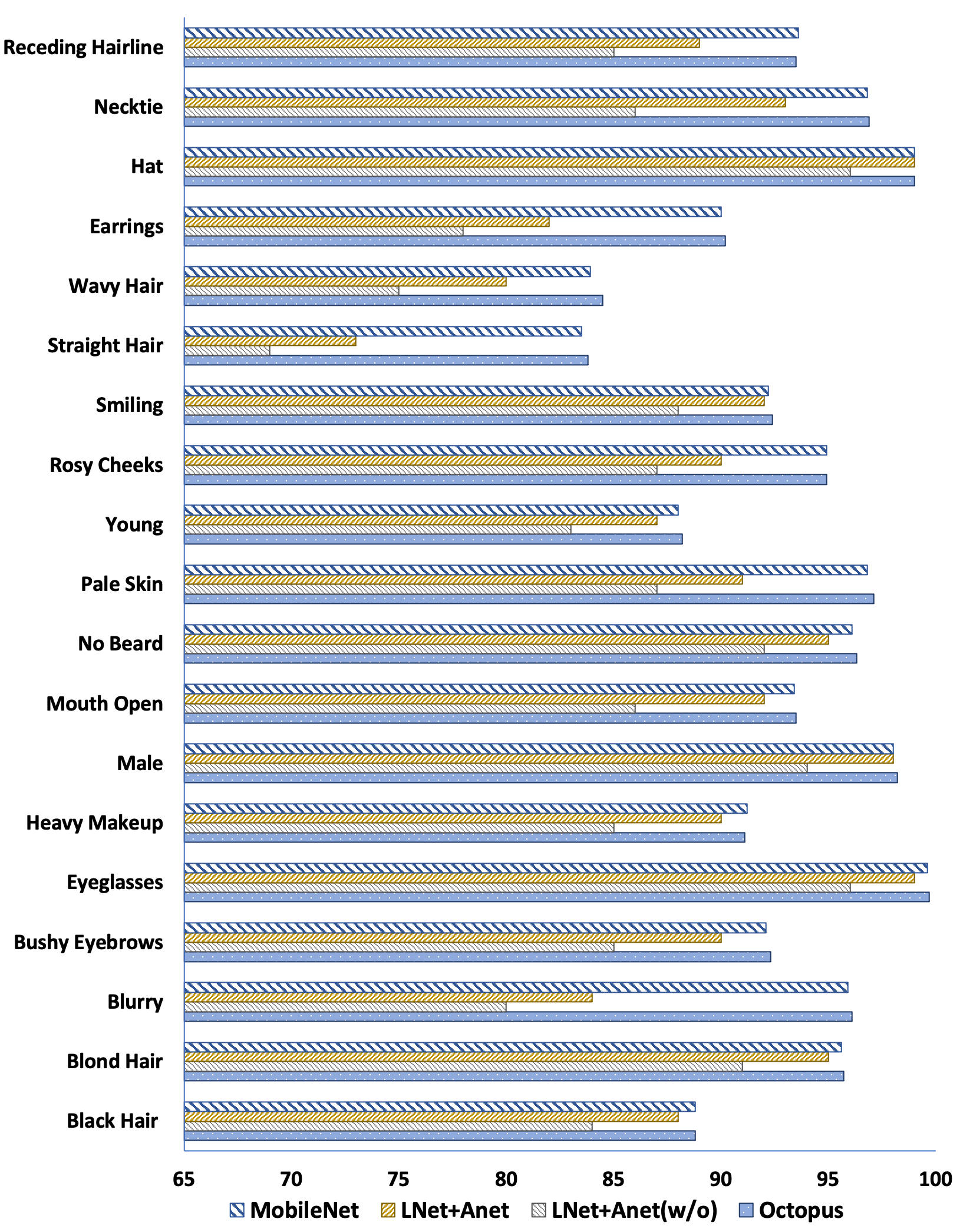}
	\caption{Accuracy of multi-tasks evaluation.}
	\label{fig:multi}
\end{figure}
\subsection{Performance under varying settings}
\label{sec:evasetting}
When the fraction of data used for training the autoencoder is increased, the performance of the autoencoder will increase. For example, the performance of the model increases by around 12-15\% when increasing the fraction of ATD from 10\% to 20\% on the image datasets, with respect to Inception Score, accounting for 90\% performance of the model trained on the entire training data (ATD is 100\%). One promising solution to loosen the assumption on the amount of publicly available data to train the autoencoder is data augmentation, such as Differentiable Augmentation (DiffAugment) \cite{zhao2020differentiable}, where the performance of the model trained using only 100 images for augmentation could match that trained on the entire dataset. One of our next step studies is to leverage data augmentation and transfer learning to address the assumption on publicly available data.  

\subsection{Evaluation on time overheads}
\label{sec:evatime}
To further illustrate the performance of Octopus, we have evaluated the time of latent code inference, downstream task training, and once-off autoencoder training.  
Taking the centralized CelebA dataset (128*128 size) as an example, the training time for a VGG-16 based classifier of gender from scratch will take more than 1 hour using one GPU with three epochs, which could be considered as the best performance limit for federated learning as well. Training time grows approximately linearly with the number of downstream tasks. However, with the same computing resources, the training time of the classifier with three linear classifier layers on the centralized latent codes is only within 5 min toward similar accuracy.  
Although the auto-encoder training takes hours, it is a once-off procedure on the server-side for multiple downstream tasks that could be carried out in parallel across several GPUs or by using the pre-trained model for transfer learning to reduce the training time further.
Given the once-off trained autoencoder, the time for extracting and disentangling latent code of each input sample is less than 0.3s for MNIST and CelebA (0.15s and 0.29s, respectively) for the CPU-Only (i5 8200 8GB) end-user. With the distributed pre-trained feature extractor, the training of multiple downstream tasks can be considered transfer learning, resulting in reduced training time and lower generalization error for the neural network model.

\section{Related Work}
\subsection{Federated learning and performance}
Vast amounts of data are necessary to deliver high-quality machine learning models \cite{vaswani2019fast}.
The conventional framework in machine learning considers a centralized dataset concocted in the integrated scheme. However, real-world data is usually decentralized across multiple parties. 
Various research works have been conducted in the last few years to apply Federated Learning (FL) to train machine learning models collaboratively while keeping the data decentralized  \cite{vogels2019powersgd,konevcny2016federateda,konevcny2016federatedb}.

Generally, a federated learning scheme has fundamentally four principal operations in an iterative manner: 
\begin{itemize}
\item Sampling Clients. Desired participants are selected by the server from a pool of clients;
\item Broadcasting Parameters. The parameters of the initialized commonly-shared model are broadcasted from the server to each selected clients;
\item Updating Local Parameters: Each participant retrains the broadcasting model in a parallel manner using the local data;
\item Aggregating Global Model: The server will aggregate the gathered uploaded parameters from each client, resulting in the updated global model. 
\end{itemize}

Being a distributed system, federated learning also encounters several challenges in terms of communication, system and data heterogeneity, as well as privatization and security. Several FL architectures are developed to address either the accuracy and computational costs or other challenges in the FL field \cite{li2020federated,aledhari2020federated,mothukuri2021survey,yang2019federated,mammen2021federated}. 

\textcolor{black}{
\textbf{Communication performance.} Even if the models are commonly less expensive to be transmitted than raw data gathering in centralized learning mechanisms, the communication burden is one of the main bottlenecks for federated learning. 
The gradient updates, frequently communicated between nodes in distributed and federated learning, are an inevitable burden, resulting in expensive communication, and scalability constraints  \cite{lin2017deep,zhang2017zipml}. Existing solutions proposed to address the communication overheads are either data compression \cite{konevcny2016federated} or by restricting only the relevant outputs by the clients \cite{hsieh2017gaia,luping2019cmfl} to be uploaded to the parameter server. 
}

\textcolor{black}{
\textbf{Heterogeneity performance.} The heterogeneity of the data and system in the network, as well as the non-identical nature of the distributed data, impact the performance of the federated learning approaches \cite{mcmahan2017communication,li2020federated}. Due to lack of global data, there is spatial heterogeneity such that each node may have a data value bias and data size variation with respect to the general population, and temporal heterogeneity such that the distribution of each local dataset may vary with time. 
FedAvg is proposed as a solution to address the heterogeneity, however, it is still not robust enough to deal with system heterogeneity. Enhancements on the model aggregation approaches have also been explored to solve the heterogeneity \cite{bonawitz2019towards,liu2020privacy}. 
}

\textcolor{black}{
\textbf{Privacy and security performance.} Even though the raw user data does not leave the local clients in FL, communicating model updates can also reveal sensitive information, requiring additional privacy-preserving strategies, such as secure computation, differential privacy schemes, or running in a trusted execution environment. 
One common privacy-preserving methodology is based on differential privacy to provide a certified privacy guarantee. The differentially private federated learning scheme \cite{geyer2017differentially,zeng2021differentially} protects a user's privacy by adding noise to the clipped model parameters before model aggregation at the cost of model accuracy. Besides, it is time- and resource-consuming to conduct perturbation or encryption for a huge amount of parameters with limited resources. Moreover, they all fail when the server is distrusted. 
There are also alternative privacy-preserving approaches to communicate gradients without incurring accuracy loss, such as secure aggregation \cite{bonawitz2017practical}. 
Secure aggregation protocols only require the averages of the local model weights to be computed from subsets of participants to compute SGD and global model updates.  
SMC-based techniques only work with honest participants. There is no guarantee that the protocol will be available or correct for Byzantine participants, particularly if they collude with a malicious server to reveal the inputs of a targeted client \cite{truong2021privacy}. 
}

\textcolor{black}{
Lack of access to global training data can cause security concerns, such as data poisoning or backdoor attacks. As the federated learning scheme is a distributed system, it might be even harder to recognize the misbehaving devices. 
The attacker can poison the data either by directly inserting poisoned data to the targeted device or injecting poisoned data through other devices \cite{sun2020data}. The adversaries could influence either the prediction of a subset of classes or the global model accuracy \cite{tolpegin2020data}. To protect the confidentiality of the training data, the aggregator has no visibility into how these updates are generated. A model-poisoning attack is proposed as a more powerful attack than poisoning attacks that target only the training data \cite{bagdasaryan2020backdoor}. 
As secure aggregation prevents the server from examining individual models, it is hard to detect the poisoning model \cite{truong2021privacy}. 
In contrast, Octopus's global data collection methods are capable of handling such erratic and malicious input via prepossessing from a global perspective. 
}

\subsection{Highlight the difference}
Even though some variants of federated learning, such as federated averaging, could perform more local computation and less communication, three significant challenges need to be considered to promote federated learning being widely adopted by the industry, i.e., heterogeneity, communication efficiency, and privacy concerns.
Due to the data heterogeneity, how FedAvg aggregates device models generates a bias in the final model, resulting in obvious accuracy degradation. 
Even sampling SGD in parallel on a small subset of devices, the total communication scale is still large. 
Consequently, different extra optimization mechanisms are required to address the communication complexities, data or system heterogeneity, and privacy concerns, resulting in further computation or accuracy degradation.
In contrast, Octopus addresses the heterogeneity, communication efficiency, and privacy challenges into one single task with less cost.
We highlight the differences between Octopus and existing approaches concerning these three challenges in the following parts.
\subsubsection{Octopus v.s. Existing propositions for heterogeneity}
Generally, FL faces heterogeneity challenges such as data heterogeneity derived from the non-IID distribution of data and personalization of models specific to customization. With sufficiently large private datasets and non-IID data distribution, local models perform better than commonly trained global models, and then clients have no incentive to participate in federated learning. 

\textcolor{black}{
The open question regarding model personalization is to define the circumstances under which commonly shared global models could be useful for local data. As a result of a lack of validation data, most existing studies can only assess the performance of the global model based on aggregate data instead of measuring its performance with respect to individual clients. To determine the performance of the common model, it may be necessary to perform local training several times. 
In addition, personalized models may not achieve the same performance as local models, particularly in the presence of differential privacy or robust aggregation.
Octopus could address these issues since the local data representations are available for global model training and validation, and no extra privacy-preserving mechanisms are required. 
To address non-IID data in federated learning, existing proposals include building a globally integrated database consisting of a small selection of data from each client \cite{zhao2018federated}, attaching a proximal constraint to the local optimization to mitigate the effects of variable local updates \cite{li2018federated}, or modifying the global optimization cost-function or weighting the aggregation at the server to deal with dissimilarities among data distributions \cite{laguel2020device,yeganeh2020inverse}. 
However, besides the extra cost of computation and accuracy degradation, these solutions also have challenges regarding communication costs, security, and privacy \cite{kairouz2019advances,jamali2021federated}, and the effect on large-scale real-world data has never been fully explored.  
Further, clustering is also used to handle the heterogeneity and personalization, conducted after the common model has been converged as a post-processing technique, enabling clients to develop more customized models. 
However, all clients are needed to engage in each round with a centralized clustering procedure, which is computationally infeasible, especially for bandwidth-limited communication channels. By examining individual clients' complete and clean weight updates, introducing a clustering process may improve performance of FL in some non-IID settings. However, it is also necessary to implement further privacy-preserving strategies.  Such noise has a negative impact on obtaining good clustering of clients and also causes degradation of the performance. 
}

\textcolor{black}{
As Octopus can be considered the global data collection and the model is trained on the collected data representations at the server-side, there is no concern about data heterogeneity from the non-IID distribution.  
}

\textcolor{black}{
Moreover, splitting data among devices in the FL makes it impossible to estimate the validation performance of the model right away, since we only have access to a few devices at a time. Therefore, decisions like which models to flag for early stopping could be noisy and not take into account all of the available validation signals. 
Due to the fact that data is separated across devices in the FL, and only a small number of devices are available at a time, it is difficult to measure the trained model's validation performance directly. Thus, making a decision regarding an early stop is not possible. Octopus, in contrast, can fully incorporate all the available validation signals for the model training. 
In addition, these approaches aiming at the heterogeneity often ignore the privacy concerns, introducing extra cost and performance degradation after adopting privacy-preserving mechanisms.
}

\subsubsection{Octopus v.s. Gradient Quantization and Sparsification for Communication Efficiency}
\textcolor{black}{
Generally, gradient quantization and sparsification are two commonly adopted approaches to address the communication efficiency of FL \cite{sun2020adaptive,li2021talk,xu2020ternary,lin2017deep,kairouz2019advances,sattler2019robust}. 
In spite of the fact that quantization incurs rounding errors, existing research studies have demonstrated that this approach results in good model convergence even at the expense of a greater number of communication rounds \cite{lin2017deep,sun2020adaptive}. The learning rate, for example, should be lower than expected when the weights are quantized to address the hard-to-converge problem. Additionally, deeper compression will often result in greater distortion of the gradients and will require more communication rounds to achieve convergence, or may even cause the model to fail to achieve convergence. As a result, additional complex procedures are necessary to ensure the gradient descent converges using quantized gradients, which are also not well generalized to certain models such as ResNet101. 
Octopus has no impact on the model's training, and it is model-agnostic. 
}

\textcolor{black}{
It is generally considered that sparsification-based approaches to communication compression compress the upstream updates, since the sparsity variations in client-side updates will always differ. 
The downstream update will not even be compressed when it comes to federated learning, in which the amount of participants is often larger than the inverse sparsity rate. 
In contrast, the communication data of Octopus is the discrete codes, which is much smaller than the compressed gradients, involving only the upstream. The downstream is a one-time download of the trained model instead of the massive iterations of gradients between server and clients. 
In addition, it remains unclear whether any quantization or sparsification method, or any combination of these, will provide the optimal trade-off between communication and accuracy in federated learning \cite{kairouz2019advances}. Research studies have shown that the more sparse or the number of clients increases, the greater the degradation of accuracy. 
Quantization-based methods include a significant loss of accuracy for some complex models, such as BERT, which is not acceptable by industry standards. 
The use of Octopus, however, has no negative influence on the accuracy of the model, but may even enhance it due to the global collection of data. 
Furthermore, existing compressed strategies always require identical compressors across all participants without consideration of heterogeneity, resulting in less flexibility. Octopus makes no assumption on the heterogeneity and exhibits more flexibility.
}

\textcolor{black}{
Besides the information loss due to the quantization or sparsification, the unpropagated gradient may contain sensitive information, resulting in privacy concerns. Therefore, extra privacy-preserving mechanisms are necessary at the cost of further computation or communication. 
For these improvements toward communication and heterogeneity, Secure Multi-Computation or Differential Privacy are also required to provide privacy protection capability in federated learning, resulting in a trade-off between cost and efficiency. Clients are required to encrypt the parameters of the models or add noise to the data before communicating with the server, inevitably leading to additional computational and accuracy degradation. 
Octopus combines communication and privacy into a single task, with no extra cost to handle the privacy. 
}

\subsubsection{Octopus v.s. Split Learning}
Besides the data partitioning and communication compression, split learning \cite{gupta2018distributed,vepakomma2018split} is proposed to decompose the deep neural network W into two portions: client-side network $W_c$ and server-side network $W_s$. The similarity between Octopus and split learning is that both enable collaborative machine learning without needing the communication of raw data with an external server and the exchange of full model parameters. 
During the common split learning procedure, each client calculates the forward pass through the network up to a cut layer. Next, the outputs of the cut layer, i.e., smashed data, are sent to an external server, which finishes the training of the remaining layers. Gradients generated from the result are then propagated back to the cut layer on the server, and then sent back to the clients for local training \cite{9060868}. 

\textcolor{black}{
There are three main advantages of Octopus compared to split learning.  
First, the communication size of the split learning is still larger than Octopus. The size of smashed data is huge for the large deep neural networks, e.g., the size of a convolutional layer is $Filter size \times Filter size \times Channel size \times Filter~number$. Besides, the communication is bidirectional, where the forward propagation and backpropagation processes continue until the network becomes trained with all the available clients and reaches its convergence. For Octopus, the communication is single directional with only 5 to 10 bits for uploading. 
In addition, the network training in the split learning scenario is done by a sequence of distributed training processes. Overall, the training is done in a relay-based manner, where the server trains with one client and then moves to another client sequentially, leading to the clients' resources being idle. This synchronization issue significantly raises the training overhead when there are several clients, not even for continuously retraining the model under time-varying data distributions and environments. For the Octopus,  multiple models could be trained in parallel as a one-time event after collecting a certain number of data, and only finetuning is required when new data comes in. 
Further, the communicated smashed data can generally reveal information about the underlying raw data. The extra cost of computations is inevitable to reduce the potential leakage, e.g., restricting the distance correlation \cite{vepakomma2018supervised} between communicated smashed data and the raw data or by applying extra privacy-preserving mechanisms. 
In contrast, Octopus applies the disentanglement strategies for data privatization, which is a preprocessing procedure during encoder training. There is no extra cost of computations required to handle information leakage during communication. 
Another concern faced by FL is vulnerability. The challenge is to identify/block adversarial clients from transmitting malicious data in order to poison the model or alter the model once it has been received by modifying its gradient/parameters before sending it back to the central server for aggregation. Octopus has no these concerns, as the model training is conducted at the server-side, and it is feasible to develop anomaly detection methods again the malicious behaviors leveraging the global data.
}

\textcolor{black}{
Consequently, federated learning mechanisms may not be practical in some real-world implementations due to the expensive communication, especially when multiple downstream tasks (e.g., classifiers) should be learned from dynamically updated and non-iid distributed data sources while providing local privatization. 
To bridge the gap between practice and efficiency, we explore an alternative distributed learning scheme to address the limitations of both federated learning and centralized learning, while enhancing their advantages. To the best of our knowledge, this paper is the first study to address communication overhead via latent compression, leveraging global data while providing local privatization of local data without additional cost due to encryption or perturbation.
}

\section{Conclusion and Discussion}
We introduce OCTOPUS as a new family of distributed computing frameworks for collecting and processing distributed data while optimizing communication aspects, computing aspects, and storage performance in one single solution. 
We demonstrate that OCTOPUS can encode and transmit high-dimensional data via compressed latent codes, thus privatizing local data without additional encryption or perturbation, while taking advantage of global data with fewer storage and computation resources. 
Our experiments demonstrate that OCTOPUS achieves likelihoods that are almost as good as centralized learning, but with the added advantages of privatization and efficient communication. 
Our study provides evidence that OCTOPUS is an efficient distributed learning scheme that can successfully address communication, computing, and storage performance, as well as the privatization of local data under one single task. In this manner, small devices such as mobile phones and home assistants (for example, Google home mini) can be used to run a lightweight learning algorithm to privatize data under various settings on the fly.

\textcolor{black}{
In this work, we consider the case of single-modal privacy, where private information is derived from the single-modal input. For example, the speaker information is derived from the waveform of the voice. Additionally, there are multimodal privacy cases in which privacy can be released when the data is in multimodal forms, and it is not easy to distinguish these multimodal public and private components. An example of speaker information is not only the waveform of the voice but also the transcript of the voice recording. As we apply attribute grouping as prepossessing to learn the disentangled strategies, the problem of distinguishing public/private components is transferred to multimodal grouping, which would be one of our future works. 
Additionally, we will explore the use of other types of auto-encoder to extract latent variables and the use of knowledge distillation to help mitigate the reduction in the model's accuracy. As a result, we focus on single-domain data at this point and will move to multi-domain data (ImageNet) in the future, where the bottleneck is data scale and computation sources, rather than method.
In addition, we consider incorporating existing communication efficiency strategies into Octopus in the future. 
}

\textcolor{black}{
With Octopus, the primary benefit is providing distributed feature encoding with disentanglement to address heterogeneity and communication efficiency, while exploring the possibility of data privatization control at no additional cost. 
For data privacy, we focus on providing local users with a level of control by separating the customized private component from the non-private component prior to sharing data, rather than releasing raw information (or its derived equivalent) with perturbation or encryption.  
Furthermore, such privatization control is directly associated with strategies designed to address expensive communication and limitations related to the heterogeneity of distributed data sources and inaccessibility to global data, with no additional privacy-preserving mechanisms incorporated. 
\newline
We conducted privacy analyses in Section 2.7 Privacy Analysis and Computational Adversary. The encoder provides a non-linear transformation as the first privacy-preserving technique. We then discussed how the IN and VQ procedures could be implemented as the second filtering-based privacy-preserving technique. 
Furthermore, we demonstrated the possibility that Octopus can also perform low-cost generation toward anonymity-based privacy-preserving techniques in Section 3.3. As an example, K anonymous samples could be generated by replacing the private component with K times of random noise. 
\newline
Our experiments demonstrated that OCTOPUS provides a higher level of data privatization control from two perspectives. 
(i) Recognition accuracy of sensitive information. Classification accuracy for the private component via OCTOPUS is reduced from $>85\%$ to $<10\%$, while the filtered data can still serve the desired classification tasks at the cost of only a small utility accuracy (4\% drop). 
(ii) We evaluated the privacy leakage of both private and public components in terms of the conditional entropy in bits against a strong computational adversary. 
We have demonstrated that the proposed OCTOPUS with no additional privacy-preserving mechanism performs as well as or better than the FL with differential privacy protection in terms of privatization control. In addition, if additional privacy-preserving approaches are applied to the latent codes, more robust protection can be achieved. 
Another advantage of OCTOPUS is the reduction of computing costs associated with perturbation or encryption of compressed latent codes rather than raw data or high-density equivalents such as gradients. We will explore this direction in future work. 
}
\bibliographystyle{IEEEtran}
{
\bibliography{egbib}
\renewcommand{\baselinestretch}{0.5}
}

\appendices
\section{Architectures and Hyperparameters}
The encoder architecture consists of Conv2D layers and Relu activation (Conv1D for Speech), followed by a BatchNormalization layer and a few Resnet-Blocks to prevent gradient vanishing. 
The public component representation is produced by the instance normalization (IN) and VQ layers that find the nearest embedding from a learned codebook for the input. 
The sensitive component representation is the average difference between continuous and discrete vectors in each batch of the same group.

\textcolor{black}{
For the models used for speech data evaluation, the speaker classifier is composed of three Conv1D layers with 256 hidden nodes followed by a fully connected layer, and the word error rate (WER) is evaluated by using the IBM Watson Cloud Speech-to-Text API. 
For the image dataset, the classifier for MNIST uses the LeNet (3 convolutional layers, 2 subsampling layers and 2 fully connected layers). For the CelebA data, the classifier is adopted the VGG-16 \cite{simonyan2014very}. Besides, a linear classifier is used for the classification of latent codes, consisting of three linear classifier layers sequentially and the Sigmoid activation function.
}
\begin{table}[!htb]
\scriptsize
\centering
	\setlength{\abovecaptionskip}{-0.05cm}
	\setlength{\belowcaptionskip}{-0.1cm}
\caption{Default Hyperparameter Settings}
\begin{tabular}{ll}
\hline
Parameter         & Default Value (MNIST, Face, Speech)      \\ \hline
Input size        & 32,128,120                               \\
Latent size       & 40,128,256                               \\
$\beta_1,\beta_2$ & 0.9, 0.999                               \\
Batch size        & 100                                      \\
Codebook size     & (k=10, d=64), (k=64, d=64),(k=256, d=64) \\
Training steps    & 20, 100000,25000                         \\
Learning rate     & 0.001, ADAM                              \\
$\lambda$         & 0.01                                     \\ \hline
\end{tabular}
\end{table}



%





\begin{IEEEbiographynophoto}{Shuo Wang}
\scriptsize
Dr. Shuo Wang is a Research Scientist at the CSIRO's Data61 and Cybersecurity CRC. He completed his Ph.D. at the University of Melbourne. His main research interests include: 
Trust AI: robustness and reliability of deep neural networks;
Computer security and privacy in systems, networking, and databases.
\end{IEEEbiographynophoto}
\begin{IEEEbiographynophoto}{Surya Nepal}
\scriptsize
Dr. Surya Nepal is a group leader and Senior Principal Research Scientist working on trust and security aspects of Web Services at CSIRO's Data61. 
\end{IEEEbiographynophoto}
\begin{IEEEbiographynophoto}{Kristen Moore}%
\scriptsize
Dr. Kristen Moore is a Senior Research Scientist at Data61 and the Cyber Security CRC, working on cyber deception. 
\end{IEEEbiographynophoto}
\begin{IEEEbiographynophoto}{Marthie Grobler}
\scriptsize
Dr. Marthie Grobler currently holds a position as Principal Research Scientist at CSIRO's Data61 where she leads the human-centric cybersecurity team.
\end{IEEEbiographynophoto}
\begin{IEEEbiographynophoto}{Carsten Rudolph}
\scriptsize
Carsten Rudolph is an Associate Professor of the Faculty of IT at Monash University, Australia, and Director of the Oceania Cyber Security Centre OCSC. 
\end{IEEEbiographynophoto}
\begin{IEEEbiographynophoto}{Alsharif Abuadbba}
\scriptsize
Dr. Alsharif Abuadbba is Senior Research Scientist at CSIRO's Data61, Australia. 
He has joined Data61 Distributed System Security group early 2019 as a Research Scientist and Cybersecurity CRC fellow.
\end{IEEEbiographynophoto}

\end{document}